%% file: main.tex
\documentclass{article}
\usepackage{iclr2027_conference,times}
\usepackage{graphicx,booktabs,longtable}
\usepackage{amsmath,amssymb,amsthm,mathtools}
\usepackage[section]{placeins}
\usepackage{float}
\AddToHook{cmd/subsection/before}{\FloatBarrier}
\usepackage{flafter,xcolor,hyperref}
\hypersetup{pdftitle={Balancing Early Performance Sacrifices with Long-Term Gains: Scaling Learning-Rate Warmup Duration Across Training Horizons},pdfsubject={Balancing Early Performance Sacrifices with Long-Term Gains: Scaling Learning-Rate Warmup Duration Across Training Horizons}}
\usepackage[capitalize,noabbrev]{cleveref}
\usepackage{url}
\usepackage{etoc}
\graphicspath{{figures/}}
\newtheorem{theorem}{Theorem}
\newtheorem{proposition}{Proposition}

\newtheorem{assumption}{Assumption}

\usepackage[breakable]{tcolorbox}\definecolor{figbg}{HTML}{EDE6DC}\makeatletter\@for\env:=theorem,proposition,lemma,corollary,assumption,remark\do{\expandafter\tcolorboxenvironment\expandafter{\env}{breakable,colback=figbg,colframe=figbg,boxrule=0pt,arc=0pt,left=4pt,right=4pt,top=2pt,bottom=2pt}}\makeatother
\makeatletter
\newcommand{\equationlabelalias}[1]{\begingroup\protected@edef\@currentlabel{\p@equation\theequation}\edef\cref@currentlabel{[equation][\number\value{equation}][]\theequation}\label[equation]{#1}\endgroup}
\makeatother

\newcommand{\Wref}{W_{\mathrm{ref}}}

\usepackage{xcolor}
\definecolor{figcyan}{HTML}{17BECF}
\hypersetup{colorlinks=true,allcolors=figcyan}
\title{Balancing Early Performance Sacrifices with Long-Term Gains: Scaling Learning-Rate Warmup Duration Across Training Horizons
}
\usepackage{etoolbox}
\AtBeginEnvironment{table}{\setlength{\belowcaptionskip}{4pt}}
\author{Kristi Topollai \\
New York University\\
\texttt{kt2664@nyu.edu} \And
  Anna Choromanska \\
  New York University \\
  \texttt{ac5455@nyu.edu} \\
}

\iclrfinalcopy
\begin{document}
\maketitle
\lhead{}                            
\etocdepthtag{main}
\maketitle

\begin{abstract}
Learning-rate warmup is a standard technique in language-model training, yet its duration remains largely heuristic. Common approaches use either a fixed number of updates or a fixed fraction of the training horizon, two choices that imply very different scaling as training gets longer. When should warmup stay fixed, and when should it grow with the horizon? We address this question with a quadratic model whose modes respond differently to the peak learning rate. Warmup slows progress in directions that already contract well at the peak rate, but can remove persistent error in directions near the stability edge, with higher peak rates shifting the balance toward longer warmup durations. This yields a compact horizon scaling law that captures regimes ranging from essentially no warmup, through fixed-duration warmup, to durations that grow with the training horizon, and explains how the preferred regime changes with peak learning rate. Because the law captures the tradeoff between giving up early progress and improving the trajectory that follows, it can be fit using shorter runs and used to predict warmup at substantially longer horizons. Together, our results explain several familiar properties of warmup through a single tradeoff and suggest treating warmup duration as a horizon-dependent hyperparameter rather than a fixed training heuristic.
\end{abstract}

\section{Introduction}
\label{sec:introduction}

Learning-rate warmup is ubiquitous in language-model training, yet there is little agreement on how long it should last. Some training setups rely on a fixed number of warmup steps \citep{vaswani2017attention,groeneveld2024olmo,kaplan2020scaling}, while others use a fixed fraction of the training horizon \citep{black2022neox}. These choices suggest different views of warmup: a fixed duration treats it as an early stabilization phase, whereas a horizon-dependent duration suggests a longer-term tradeoff whose optimal balance changes with the training budget. Why should an intervention applied only at the start of training depend on when training stops?

\begin{figure}[!t]
\centering
\includegraphics[width=0.9\textwidth]{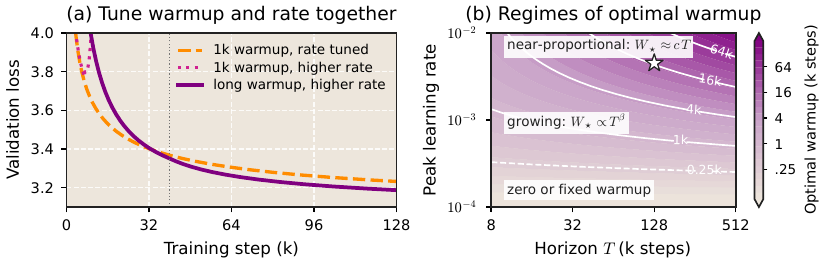}
\vspace{-1em}
\caption{\textbf{Overview.}
\textbf{(a)} Longer warmup can trade slower early progress for better performance at aggressive peak learning rates.
\textbf{(b)} Schematic of the loss-optimal warmup regimes, from little or bounded warmup to durations that grow with the training horizon, depending on peak learning rate. The star marks the approximate location of the best measured runs}
\vspace{-1.5em}
\label{fig:overview}
\end{figure}

Warmup is standard in large-batch and language-model training, including frontier-scale systems \citep{grattafiori2024llama3,deepseek2024v3,kimi2025k2}. Its duration must balance two competing effects: too little warmup can destabilize training at larger peak learning rates, while too much sacrifices early optimization progress. This tradeoff also couples warmup duration to the peak learning rate itself: a rate that appears suboptimal under short warmup can become competitive once longer warmup is allowed, so fixing warmup in advance can change which learning rate appears best.

Prior work explains warmup through adaptive-optimizer statistics and update magnitudes \citep{liu2020radam,ma2021adequacy}, loss curvature and access to larger subsequent learning rates \citep{gilmer2022loss,kalra2024warmup}, and early changes in parameters and representations \citep{gotmare2019closer,kosson2024warmup}. Recent theory also derives warmup-like schedules from suboptimality-dependent smoothness \citep{liu2025warmupconvergence,alimisis2025warmup,riabinin2026warmup}. However, explaining why warmup helps does not determine its optimal duration. That requires balancing the persistent benefit of a longer warmup against the progress lost to smaller initial steps.

We begin with a broad, systematic empirical study of warmup duration, varying peak learning rate and warmup across training horizons to understand how the loss-optimal duration behaves in practice. These experiments show that warmup can remain nearly fixed at lower peak rates, grow substantially with the training horizon at larger rates, and change which learning rates are competitive. To understand these patterns, we identify a simple quadratic model that reproduces them and explains why higher peak rates favor longer warmup. The resulting tradeoff between slower early progress and reduced persistent error motivates a compact loss law over warmup duration and training horizon, which we fit to short training trajectories and use to predict warmup choices at longer horizons, including from as few as three short runs.

Our results place warmup duration within the scaling-law view of training. Like other optimization hyperparameters \citep{steplaw2025,zhang2025critical,bergsma2025powerlines}, its preferred value changes systematically with the training budget rather than remaining fixed. Prior work has studied how loss \citep{kaplan2020scaling,hoffmann2022chinchilla} and learning-rate optima \citep{bjorck2025horizonlr,filatov2024time} vary with model size, compute, and training horizon. We show that warmup duration has its own horizon dependence, shaped by the peak learning rate. Modeling this dependence turns warmup from a startup heuristic into a hyperparameter with its own scaling law.

\section{Warmup regimes across learning rates and training horizons}
\label{sec:phenomenon}
\subsection{Experimental design}

Our starting point is an empirical study of warmup duration in Llama-style language models \citep{grattafiori2024llama3}. We vary peak learning rate and warmup duration under linear warmup followed by a constant rate, and evaluate validation loss across training horizons. Let $W$ denote warmup duration and $T$ the evaluation horizon; for fixed model size and peak rate, varying $W$ defines a \emph{family} of runs. The ranges and training protocol are given in \cref{app:observations}.

We use a warmup--stable schedule because one trajectory can be evaluated at many horizons, whereas schedules whose decay depends on the final horizon generally require separate runs. Warmup--stable is also the reusable prefix of warmup--stable--decay (WSD) \citep{hagele2024beyond,wen2025river}, a popular schedule in recent large-scale systems \citep{deepseek2024v3,kimi2025k2}. Existing WSD theoretical analyses focus mainly on the stable and decay phases \citep{wen2025river,schaipp2025surprising}. Schedule-aware scaling laws have primarily focused on the post-warmup decay phase rather than on how the initial warmup duration itself should scale with training horizon \citep{tissue2024annealing,luo2025mpl}. By holding the post-warmup learning rate constant, we isolate this question. In \Cref{app:decay-branching}, we show that appending decay to
warmup--stable trajectories largely preserves the ordering of warmup durations.

\subsection{From short warmup to horizon-dependent warmup}
At conservative peak rates (\cref{fig:story}a), little or no warmup can suffice and the preferred duration changes little with training horizon. At larger rates, the lowest-loss region shifts toward longer warmups as training continues, so the duration that is best early need not remain best later; depending on the family, this growth ranges from weak horizon dependence to nearly proportional.

\begin{figure}[!htbp]
\centering
\vspace{-0.5em}
\includegraphics[width=0.9\textwidth]{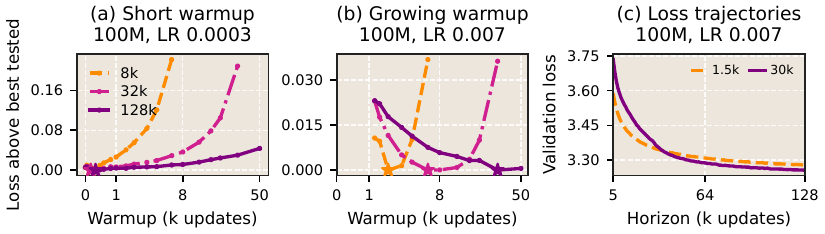}
\vspace{-1em}
\caption{\textbf{Peak learning rate changes how optimal warmup evolves with horizon.}
\textbf{(a,b)} Lower rates favor short warmup, while higher rates increasingly favor longer durations as training continues.
\textbf{(c)} At high peak rate, a longer warmup can start worse but become better later, so its benefit persists beyond the warmup duration.}
\label{fig:story}
\end{figure}

This horizon dependence is not solely due to avoiding unstable runs. In \cref{fig:story}c, two successful warmup schedules reverse their ordering after both have ended: the longer warmup is initially worse but becomes better later. Thus, even among stable runs, the loss-optimal warmup can depend on the evaluation horizon. Across the grid, we observe both nearly fixed and horizon-dependent warmup, with longer optimal warmups becoming more common at larger peak learning rates (\cref{fig:regimes}).

\begin{figure}[!htbp]
\centering
\includegraphics[width=\textwidth]{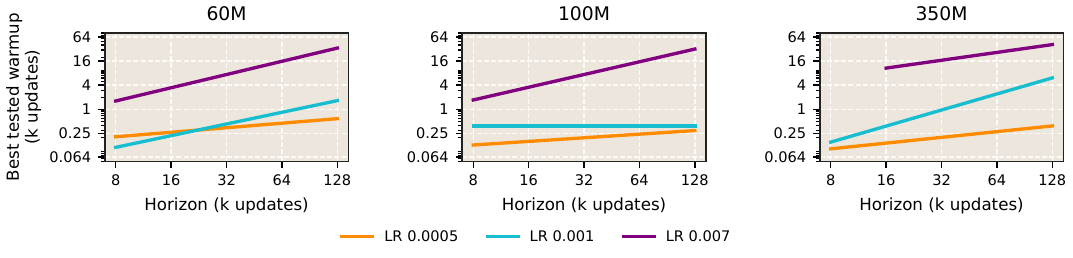}
\vspace{-2em}
\caption{\textbf{Fixed and horizon-dependent warmup regimes coexist.}
Lines summarize how the measured loss-optimal warmup changes with training horizon within each family.}
\vspace{-1em}
\label{fig:regimes}
\end{figure}
\paragraph{Warmup needed to avoid training failure.}
\label{sec:startup-stability}

Warmup can make otherwise failing configurations trainable. We measure the \emph{shortest successful warmup}: the smallest tested duration with at least one completed run that converges. This threshold increases with higher peak learning rates, smaller batches, and larger models (\cref{fig:stability-requirements}). The warmup required for a successful start therefore depends strongly on the training regime. However, avoiding failure is only one role of warmup. The shortest successful warmup is often far shorter than the loss-optimal one, so stability and final-loss optimization impose different requirements.

\begin{figure}[!htbp]
\centering
\includegraphics[width=\textwidth]{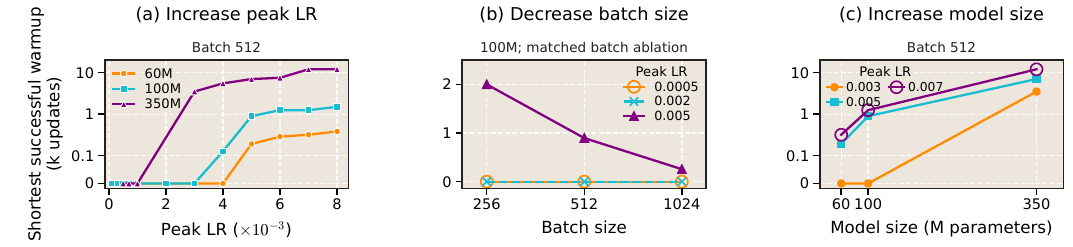}
\vspace{-2em}
\caption{\textbf{Shortest successful warmup durations.}
Shortest tested successful warmup versus \textbf{(a)} peak learning rate,
\textbf{(b)} batch size, and \textbf{(c)} model size.}
\label{fig:stability-requirements}
\end{figure}

\subsection{Long warmups unlock competitive high learning rates}

The strongest horizon dependence appears at larger peak learning rates, but these are not merely pathological settings. In the best-performing region, high peak rates are often paired with long warmups (\cref{tab:lr-winners}), so restricting warmup can make otherwise competitive rates appear suboptimal.

\begin{table}[!htbp]
\centering
\setlength{\tabcolsep}{15pt}
\caption{\textbf{The best tested runs use long warmups.}
Learning rate and warmup are jointly tuned.}
\label{tab:lr-winners}
\scriptsize
\begin{tabular}{lrrrr}
Model & Peak LR & Warmup (k) & Fraction of horizon (T=128k) & Loss\\\midrule
60M & 0.004 & 30 & 23.4\% & 3.404\\
100M & 0.005 & 50 & 39.1\% & 3.246\\
350M & 0.004 & 40 & 31.3\% & 2.983\\
\bottomrule
\end{tabular}
\vspace{-1em}
\end{table}

We also compare unrestricted tuning with a short-warmup constraint, $W\leq1$k. Allowing longer warmups improves the loss attainable after tuning warmup at each peak rate and can shift the preferred learning rate (\cref{fig:joint-tuning}).

\begin{figure}[!htbp]
\centering
\includegraphics[width=\textwidth]{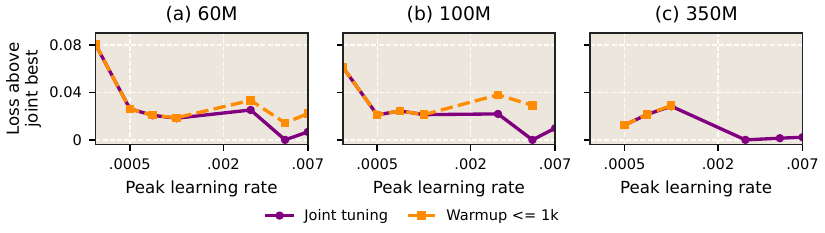}
\vspace{-2em}
\caption{\textbf{Long warmups improve the loss attainable at high peak rates.}
Validation loss at 128k with warmup tuned freely or capped at 1k. Loss is relative to the best joint configuration at each scale; capped curves end when no successful short-warmup run remains. Additional learning rates used for the stability comparison are omitted because they used a sparser warmup grid.}
\label{fig:joint-tuning}
\end{figure}

Prior work finds that the best peak learning rate often decreases as training gets longer, but those studies use different warmup rules \citep{bjorck2025horizonlr,filatov2024time}. In fact, changing warmup can itself change which learning rate works best \citep{filatov2024time}, consistent with \citet{porian2024discrepancies}, who identify warmup duration as one source of confounding in scaling results. Our results show the same effect: the warmup rule can change which learning rate looks optimal at a given training horizon.

\section{A Mechanism for Horizon-Dependent Warmup}
\label{sec:quadratic}
To understand the empirical patterns above and ultimately predict how the loss-optimal warmup should change with the training horizon, we seek a simple model that captures the key tradeoffs while remaining analytically tractable. We therefore turn to quadratic dynamics, a standard lens for studying learning-rate, batch-size, and stability effects in optimization
\citep{jain2018parallel,zhang2019nqm,bordelon2022sgd,meterez2026defense}. A quadratic is particularly useful here because different curvature directions can respond very differently to the same learning-rate schedule, allowing us to isolate how warmup can trade slower early progress for lower error later in training. Consider gradient descent on the quadratic
$$
f(\theta)-f_\star=\tfrac12(\theta-\theta_\star)^\top H(\theta-\theta_\star),
\qquad H\succeq0.
$$
In a direction of curvature $h>0$, one update with learning rate $\eta_t$ multiplies that direction's loss by $|1-\eta_t h|^2$. We call this multiplicative decrease \emph{contraction}: a smaller multiplier means that error in the direction is removed more quickly. Writing $\alpha=\eta h$, define the contraction rate at the peak learning rate and its average over a linear warmup as
\begin{equation}
\rho(\alpha)=-2\log|1-\alpha|,
\qquad
g(\alpha)=\int_0^1 \rho(\alpha u)\,\mathrm du.
\label{eq:det-phase-rates}
\end{equation}
After $n$ peak-rate updates, the loss in this direction is multiplied by $e^{-n\rho(\alpha)}$, while $g(\alpha)$ averages the contraction rates encountered as the learning rate increases from $0$ to $\eta$. Approximating the linear warmup as continuous, the loss remaining in this direction after warmup and the subsequent peak-rate phase is
\begin{equation}
M_{T,W}(\alpha)\approx
e^{-g(\alpha)W-\rho(\alpha)(T-W)}.
\label{eq:det-phase-envelope}
\end{equation}

\begin{figure}[!htbp]
\centering
\includegraphics[width=\textwidth]{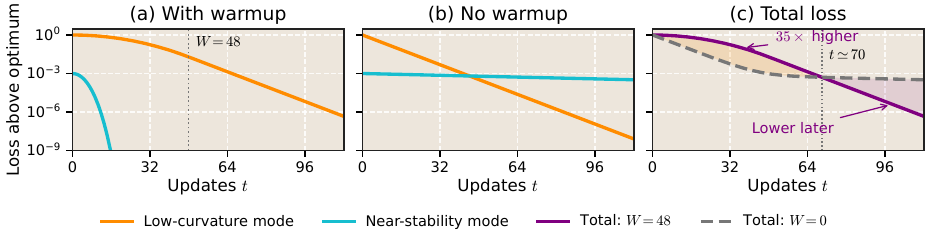}
\vspace{-2em}
\caption{\textbf{Warmup trades early progress for lower persistent error.}
\textbf{(a,b)} A longer duration slows progress in one direction but reduces error in another that improves only slowly at the peak learning rate.
\textbf{(c)} As a result, the longer warmup is worse early in training but better later, matching the ranking reversal observed empirically.}
\label{fig:det-competition}
\end{figure}

For small $\alpha$, $g(\alpha)\approx\alpha$ and
$\rho(\alpha)\approx2\alpha$: warmup slows progress. Near the upper stable
boundary, the comparison reverses. As $\alpha\uparrow2$,
$\rho(\alpha)\to0$ while $g(\alpha)\to2$. A longer warmup can remove error
that would persist at a stable peak rate.

\paragraph{Higher peak rates favor longer warmup.}
Consider two modes distinguished by how their contraction compares between warmup and the peak-rate phase. For a \emph{low-curvature mode}, $\rho_s>g_s$, so replacing peak-rate updates with warmup updates slows its decay; define $d_s:=\rho_s-g_s>0$. For a \emph{near-stability mode}, $g_e>\rho_e$, so warmup contracts it more strongly than the subsequent peak-rate phase; define $a_e:=g_e-\rho_e>0$.
Their loss is
\begin{equation}
L_{\mathrm{quad}}(W,T)
=A_s e^{-\rho_sT+d_sW}+A_e e^{-\rho_eT-a_eW},
\qquad A_s,A_e>0.
\label{eq:det-two-mode-loss}
\end{equation}
Balancing these terms gives
\begin{equation}
W_{\mathrm{quad}}^*(T,\eta)=\Pi_{[0,T]}
\left[\frac{(\rho_s-\rho_e)T+\log(A_ea_e/(A_sd_s))}{d_s+a_e}\right].
\label{eq:main-quadratic-optimum}
\end{equation}
If $0<h_s<h_e/2$ and $\eta h_e<2$, the long-horizon optimum changes at
$\eta_c=2/(h_s+h_e)$. It is zero for $\eta<\eta_c$, bounded at
$\eta=\eta_c$, and a fraction of $T$ for $\eta>\eta_c$, a fraction that
grows with $\eta$. This matches the empirical transition in
\Cref{fig:story,fig:regimes}. \Cref{app:det-quadratic} gives the proof.

\paragraph{Warmup above the stability limit.}
The same model also captures the stability trend seen empirically: as peak learning rate increases, completing training without divergence requires progressively longer warmup (\cref{fig:stability-requirements}). Once the near-stability mode crosses the stability boundary, $\alpha_e=\eta h_e>2$, the constant-rate phase amplifies its error. Writing $u_e=-\rho_e>0$, its contribution scales as
$A_e\exp\{u_eT-(g_e+u_e)W\}$. Preventing net growth requires
\begin{equation}
\frac WT\ge\kappa(\alpha_e):=
\frac{-\rho(\alpha_e)}{g(\alpha_e)-\rho(\alpha_e)}
=\frac{\alpha_e\log(\alpha_e-1)}
{\alpha_e+\log(\alpha_e-1)}.
\label{eq:main-rescue-fraction}
\end{equation}
In the continuous-warmup approximation, the minimum warmup fraction required to avoid net growth increases with the peak rate and eventually reaches one. The model therefore reproduces both empirical trends: higher peak rates favor longer loss-optimal warmup before instability, and require longer warmup to avoid net growth by a given horizon after crossing the stability edge.

\paragraph{What the quadratic isolates.}
The fixed-quadratic analysis deliberately separates the effect of warmup duration from the effect of schedule ordering. Because the update matrices commute, its horizon dependence arises purely from how different curvature directions respond to the rates encountered during warmup and at the peak rate and it does not require an ordering effect. In neural-network training, ordering can matter as well. \Cref{app:det-stochastic-complement} shows that finite-batch sampling alone makes the order of learning rates matter and
can produce the same early--late tradeoff, while evolving curvature provides another mechanism \citep{gilmer2022loss,kalra2024warmup,alimisis2025warmup}.

\section{A loss law for warmup duration}
\label{sec:model}\label{sec:surface}

Section~3 showed the basic tradeoff in two quadratic modes: longer warmup sacrifices early progress but can reduce error that persists at the peak learning rate. This simple picture is consistent with behavior observed in neural networks, where Hessians can contain isolated large-eigenvalue directions \citep{sagun2017empirical,ghorbani2019investigation} and training can operate near the gradient-descent stability boundary \citep{cohen2021eos,damian2023self}. We now use the same tradeoff to motivate a compact loss law over warmup duration and training horizon.

\paragraph{The cost of slower initial progress.}
With linear warmup of duration $W$, the accumulated learning rate by horizon $T$ is approximately
$$
\eta W/2+\eta(T-W)=\eta\tau,
\qquad
\tau:=T-\frac W2.
\label{eq:progress-clock}
$$
Thus $\tau$ measures peak-rate-equivalent updates. For small $\eta h$, modal losses decay approximately as $e^{-2\eta h\tau}$. We adopt the common spectral power-law assumption that initial loss below curvature $h$ scales as $h^p$ near zero ($p>0$) \citep{bordelon2024dynamical,velikanov2024tight}. Summing these responses yields $A_\eta\tau^{-p}$ (\cref{app:det-progress}). Longer warmup reduces $\tau$, so this term increases.

\paragraph{From persistent error to a horizon--warmup law.}
Near the stability edge, warmup can remove error that the peak-rate phase contracts only slowly. For a reference direction, $\rho_\eta=\rho(\eta h_e),\ 
g_\eta=g(\eta h_e),$ and its accumulated contraction is
\begin{equation}
C_\eta(W,T)
=\rho_\eta\tau+b_\eta W,
\qquad
b_\eta=g_\eta-\frac{\rho_\eta}{2}>0.
\label{eq:det-contraction-clock}
\end{equation}
We model the aggregate persistent error by an analogous mixture of slowly relaxing components governed by $C_\eta$, with a power-law density of effective decay rates near zero. This gives $E_\eta\propto C_\eta^{-\gamma}$ for large positive $C_\eta$. Expanding the logarithm of this power-law surrogate to first order in $\log\tau$ and $\log W$ around a representative training point yields the local product approximation
$$
E_\eta(W,T)\approx K_\eta\tau^{-q_\eta}W^{-s_\eta},
\qquad q_\eta+s_\eta=\gamma.
$$
In this local approximation, approaching the stability edge shifts more of the persistent term's dependence toward warmup; the derivation is in \cref{app:det-exponent}. Combining this term with the cost of delayed progress, and using $w_0>0$ to keep the predictor well-defined at zero warmup, gives
\begin{equation}
\boxed{
L(W,T)=L_\infty+A\tau^{-p}
+K\tau^{-q}(W+w_0)^{-s}
}
\label{eq:absolute-law}
\end{equation}
\equationlabelalias{eq:general-loss-law}
\vspace{-1em}
\subsection{Why optimal warmup depends on the horizon}
\label{sec:law}
With the coefficients fixed, longer training changes the balance between
warmup's initial cost and its lasting benefit. When $w_0\ll W_*\ll T$,
the cost of another warmup step scales as $T^{-(p+1)}$, while its benefit
scales as $T^{-q}W_*^{-(s+1)}$. Balancing these terms gives
\begin{equation}
 W_*\sim\left(\frac{2sK}{Ap}\right)^{1/(s+1)}T^\beta,
 \qquad \beta=\frac{p+1-q}{s+1}.
 \label{eq:product-regime-index}
\end{equation}
For $0<\beta<1$, the benefit fades more slowly than the cost at a fixed
warmup duration, so $W_*$ grows as a sublinear power of $T$. The full law also covers bounded or zero warmup ($\beta\le0$) and warmup proportional to $T$ ($\beta\ge1$); see \cref{app:det-law-proof} for the limits.

\subsection{Fitting the measured loss curves}
We fit $(L_\infty,A,K,p,q,s)$ separately for each family with a bounded multistart Huber fit similar to \citet{hoffmann2022chinchilla}. The same functional form captures short- and long-warmup families (\cref{fig:absolute-fits}), explains most loss variation, and selects low-regret warmups (\cref{tab:loss-metrics}). The fitted bounds keep $\beta>0$ for every family, so nearly fixed warmup at low peak rates appears through a near-zero warmup scale instead (\cref{fig:absolute-duration}b,c). All tables report RMSE and regret in $10^{-3}$ loss units, and $\pm$ is the sample standard deviation across horizons after any averaging over families.
\begin{table}[H]
\centering
\setlength{\tabcolsep}{15pt}
\caption{\textbf{Absolute-fit accuracy and selection quality.}
$R^2$ and RMSE are medians across families; regret is the mean loss above the best tested warmup. RMSE and regret are in $10^{-3}$ loss units.}
\label{tab:loss-metrics}
\scriptsize
\begin{tabular}{lrrr}
Evaluation & Median $R^2$ & RMSE & Mean regret\\\midrule
All horizons (descriptive) & 0.996 & 4.92 & $2.00 \pm 1.51$ \\
Alternating holdout & 0.996 & 4.98 & $2.05 \pm 1.61$ \\
\bottomrule\end{tabular}
\end{table}

\begin{figure}[!htbp]
\centering
\includegraphics[width=\textwidth]{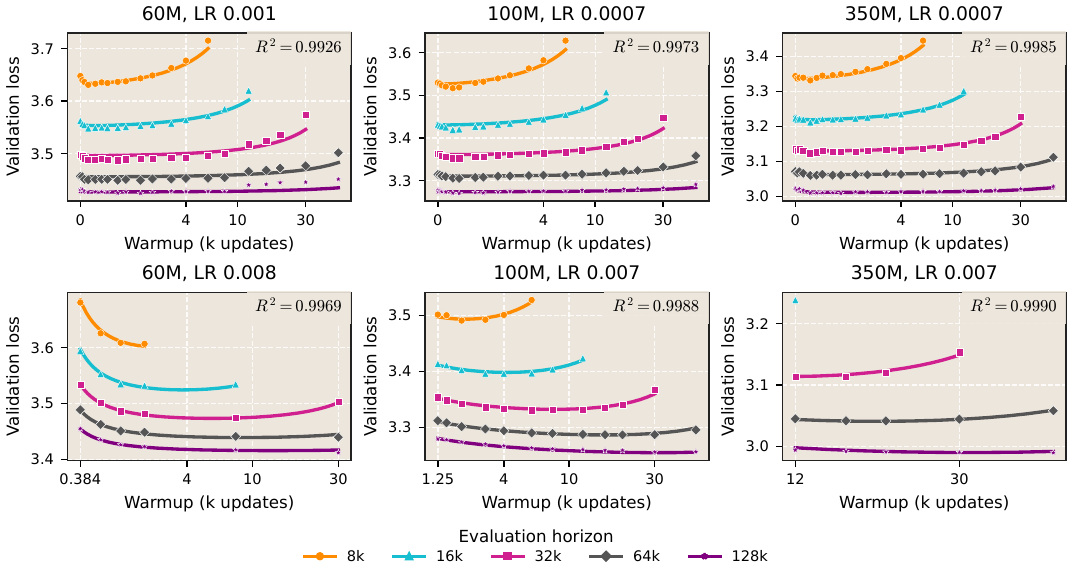}
\vspace{-2em}
\caption{\textbf{The loss law captures qualitatively different warmup regimes.}
Selected illustrative full-data fits span low and high peak learning rates across model scales.}
\label{fig:absolute-fits}
\end{figure}

\begin{figure}[!htbp]
\centering
\includegraphics[width=\textwidth]{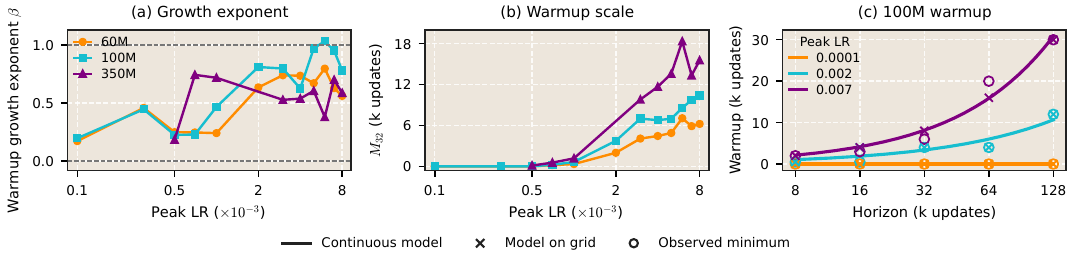}
\vspace{-2em}
 \caption{\textbf{Growth and scale summarize the fitted warmup regimes.}
\textbf{(a)} Warmup growth exponent $\beta=(p+1-q)/(s+1)$ from the fitted loss law.
\textbf{(b)} Corresponding warmup scale $M_{32}=(2sK/(Ap))^{1/(s+1)}32^\beta$ for the interior approximation $W_*\approx M_{32}(T/32)^\beta$.
\textbf{(c)} Representative 100M families comparing the continuous, model-selected and measured optimum.}

\label{fig:absolute-duration}
\end{figure}

\section{Predicting warmup at later horizons}
\label{sec:selection}\label{sec:method}

The horizon dependence in \cref{eq:absolute-law} suggests an extrapolation rule. We fit the loss law on shorter trajectories, freeze its coefficients, and use the predicted loss ordering to choose warmup at later horizons. With the full warmup grid, this tests the law. With only three short runs, it becomes a practical protocol.

\subsection{From the loss law to a selection rule}

Warmup selection depends only on differences between candidate losses. Subtracting a reference warmup therefore preserves the minimizing duration:
\begin{equation}
\Delta L(W,T;\Wref)=L(W,T)-L(\Wref,T).
\label{eq:reference-difference}
\end{equation}
With $\tau_W=T-W/2$ and $\tau_{\rm ref}=T-\Wref/2$, subtracting \cref{eq:absolute-law} gives
\begin{equation}
\boxed{
\Delta L(W,T;\Wref)=A[\tau_W^{-p}-\tau_{\rm ref}^{-p}]
+C\left[\tau_W^{-q}\left(\frac{W+w_0}{\Wref+w_0}\right)^{-s}
-\tau_{\rm ref}^{-q}\right].
}
\label{eq:decision-law}
\end{equation}
We fit $(A,C,p,q,s)$ to measured loss differences at the fit horizons and choose the available duration with the lowest predicted loss. The shortest successful warmup serves as the reference. This removes $L_\infty$ while preserving the competing effects of delayed progress and persistent error.

\subsection{Horizon extrapolation}
\label{sec:horizon-transfer}

We fit the absolute and difference forms independently using
checkpoints through 32k updates, freeze their parameters, and select warmup at
$\mathcal H=\{50\mathrm{k},75\mathrm{k},100\mathrm{k},128\mathrm{k}\}$.
We also evaluate interleaved held-out checkpoints within the fit window.
At every target horizon, all methods choose from the same retained successful
warmup candidates. For a selected duration $\widehat W_T$ and candidate set $\mathcal G_T$, we
measure
\begin{equation}
\mathcal R_T
=
L(\widehat W_T,T)
-
\min_{W\in\mathcal G_T}L(W,T)
\ge 0,
\label{eq:regret}
\end{equation}
and average regret equally across target horizons and then across families.

\begin{table}[!htbp]
\centering
\setlength{\tabcolsep}{15pt}
\caption{\textbf{Fits through 32k select low-regret warmups at later horizons.}
Mean regret in $10^{-3}$ loss units. Both fits use checkpoints through 32k
updates for the four-horizon and 128k evaluations; the alternating column
uses interleaved held-out checkpoints.}
\label{tab:policy}
\scriptsize
\begin{tabular}{lrrrr}
Selector & Parameters & Alternating & Four horizons & 128k\\\midrule
Absolute fit & 6 & $2.05 \pm 1.61$ & $2.26 \pm 0.66$ & $3.05$ \\
Difference fit & 5 & $\boldsymbol{1.45} \pm 2.22$ & $\boldsymbol{1.58} \pm 0.77$ & $\boldsymbol{2.61}$ \\
1k target & 0 & $15.47 \pm 10.80$ & $13.79 \pm 0.22$ & $13.56$ \\
10\% target & 0 & $13.39 \pm 4.10$ & $12.58 \pm 0.90$ & $11.72$ \\
\bottomrule\end{tabular}

\end{table}

\begin{figure}[!htbp]
\centering
\vspace{-1em}
\includegraphics[width=0.8\textwidth]{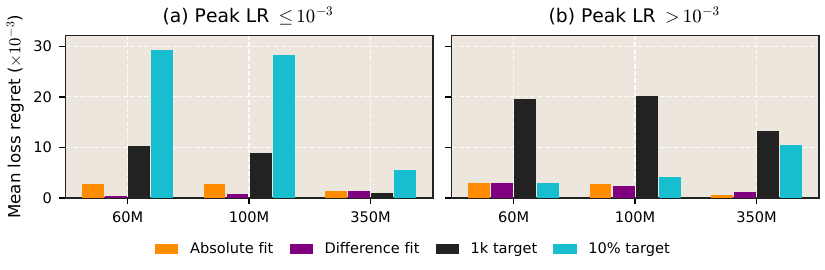}
\vspace{-1em}
\caption{\textbf{Warmup selection extrapolates beyond the fit window.}
Bars average regret over 50k, 75k, 100k, and 128k after fitting through
32k, then equally over all fitted peak learning rates within each scale
and indicated range.}
\label{fig:horizon-regret}
\end{figure}

The fixed 1k and 10\% targets give a scale for regret. On this scale both fits select low-regret warmups within the fit window and at later horizons, with the difference fit best in every column of \cref{tab:policy}. The full grid thus measures how faithfully the fitted law carries the horizon dependence of the loss forward, and \cref{sec:practical-extrapolation} turns this into a practical protocol.

\subsection{Practical extrapolation from three short runs}
\label{sec:practical-extrapolation}

We next replace the full warmup grid with only three separated, stable warmup trajectories per family. Using their shared post-warmup checkpoints through 20k or 32k, we fit the loss law, freeze its coefficients, and select
warmup at the same four target horizons. We retain both fitting methods and use the absolute fit as the primary practical predictor. Both loss-law fits outperform the baselines common in practice and naive extrapolation of the best fitting run (\Cref{tab:practical-extrapolation}), so horizon extrapolation does not need a dense grid once a few stable fitting runs cover the relevant warmup range. \Cref{fig:three-pilot-detail} breaks down the absolute-fit and baseline results at 32k by scale and peak rate. The absolute fit has lower mean regret at both fit horizons. With three fitting runs, the difference fit uses one as a reference, leaving only two curves of loss differences to constrain its five parameters. The absolute fit instead retains all three loss curves, and its stronger performance in the three-run setting motivates using this method for practical and long-horizon prediction.

\begin{table}[!htbp]
\centering
\caption{\textbf{Three short runs support longer-horizon warmup prediction.} Mean regret over 50k, 75k, 100k, and 128k, in $10^{-3}$ loss units. Fixed dur.\ and fixed frac.\ keep the best fitting-run duration or its fraction of the fit horizon.}
\label{tab:practical-extrapolation}
\scriptsize
\begin{tabular}{@{}lcccccc@{}}
Fit horizon & Absolute & Difference & Fixed dur. & Fixed frac. & 1k & 10\%\\
\midrule
20k (29 families) & $\mathbf{2.83}\pm0.16$ & $4.22\pm0.29$ & $9.31\pm1.21$ & $11.72\pm0.82$ & $14.28\pm0.30$ & $13.42\pm0.96$\\
32k (33 families) & $\mathbf{2.51}\pm0.27$ & $2.62\pm0.36$ & $7.61\pm1.12$ & $8.47\pm0.62$ & $13.79\pm0.22$ & $12.58\pm0.90$\\
\bottomrule
\end{tabular}

\end{table}

\begin{figure}[!htbp]
\centering
\vspace{-1em}
\hspace*{-1.3cm}
\includegraphics[width=\textwidth]{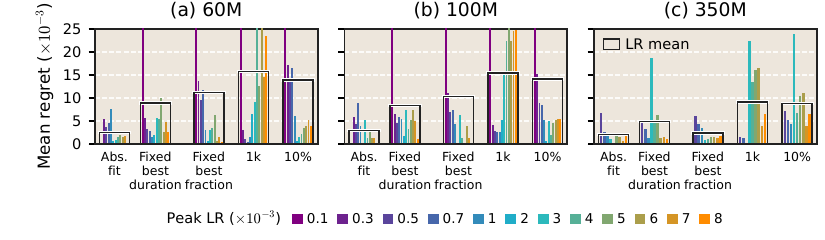}
\vspace{-1em}
\caption{\textbf{Three-run selection by scale and peak learning rate.}
Absolute-fit selection and the two baselines after fitting through 32k; regret averaged over 50k, 75k, 100k, and 128k.}
\label{fig:three-pilot-detail}
\end{figure}

\subsection{Long-horizon extrapolation}
\label{sec:long-horizon-extrapolation}

We extend the three-run procedure to horizons of up to one million updates using the 100M model at a smaller batch size. Following the three-run extrapolation, we use the absolute fit. For each peak learning rate, we fit the loss law to all three fitting trajectories using prefixes through 50k or 75k updates, freeze the fit, and predict warmup separately at 250k, 500k, and one million updates. We compare the forecasts with a fixed 1k-update warmup, which trains without divergence at every peak rate tested, and a 10\% warmup ratio. At the two higher peak rates the forecasts give the lowest loss at every horizon, and at the lowest rate they match the 1k baseline (\cref{tab:long-horizon-extrapolation}).

\begin{table}[!htbp]
\centering
\setlength{\tabcolsep}{10pt}
\caption{\textbf{Long-horizon validation-loss comparison.}
Three-run absolute-fit forecasts from 50k or 75k prefixes, compared with 1k-warmup and 10\% baselines. Horizons and predicted warmups are in thousands of updates.}
\label{tab:long-horizon-extrapolation}

\scriptsize
\begin{tabular}{lrrrrrrrrr}
                  & \multicolumn{3}{c}{LR $=0.5\times10^{-3}$}       & \multicolumn{3}{c}{LR $=2\times10^{-3}$}         & \multicolumn{3}{c}{LR $=5\times10^{-3}$}         \\ \cline{2-10} 
$T$ (k)           & 250            & 500            & 1000           & 250            & 500            & 1000           & 250            & 500            & 1000           \\ \hline
1k-warmup loss    & \textbf{3.370} & 3.337          & \textbf{3.311} & 3.433          & 3.421          & 3.402          & 3.371          & 3.347          & 3.322          \\
10\% loss         & 3.375          & 3.347          & 3.339          & 3.418          & 3.367          & 3.340          & 3.363          & 3.335          & 3.314          \\ \hline
$\widehat W_{50}$ & 0.6            & 0.7            & 0.8            & 47.7           & 77.5           & 125.9          & 129.5          & 241.7          & 451.1          \\
Loss              & 3.372          & \textbf{3.336} & \textbf{3.311} & 3.412          & 3.362          & 3.338          & 3.359          & 3.325          & \textbf{3.303} \\
$\widehat W_{75}$ & 0.8            & 1.1            & 1.2            & 56.6           & 95.2           & 160.0          & 116.0          & 213.5          & 392.9          \\
Loss              & 3.371          & 3.336          & 3.312          & \textbf{3.411} & \textbf{3.359} & \textbf{3.336} & \textbf{3.358} & \textbf{3.324} & \textbf{3.303} \\ \hline
\end{tabular}
\end{table}

\section{Conclusion}
Warmup duration is neither fixed across training budgets nor tied to a fixed fraction of the run. In our experiments, lower peak learning rates often favor short or bounded warmup, while higher rates favor durations that grow with the training horizon. A simple quadratic model explains this shift as a balance between giving up early progress and reducing error that would otherwise persist later in training. This motivates a compact loss law that captures the observed warmup regimes and can be fit using only three short trajectories to guide warmup choices at longer horizons. Our results therefore suggest treating warmup duration as a horizon-dependent hyperparameter, tuned jointly with peak learning rate and training budget.

\bibliography{references}
\bibliographystyle{iclr2027_conference}

\clearpage
\phantomsection\label{paper:appendix-start}
\input{appendix}

\end{document}

%% file: appendix.tex
\appendix
\etocdepthtag{appendix}
\begingroup
\hypersetup{linktoc=page}
\etocsettagdepth{main}{none}
\etocsettagdepth{appendix}{subsection}
\etocsetnexttocdepth{subsection}
\etocsettocstyle{\section*{Appendix contents}}{}
\tableofcontents
\endgroup

\section{Experimental and fitting protocols}
\label{app:protocol}
\subsection{Training setup, observations, and units}
\label{app:observations}
We analyze Llama-style language models trained with linear warmup followed by
a constant peak learning rate. Validation checkpoints after warmup are used
for fitting; repeated observations at the same model-size, peak-rate, warmup,
and horizon coordinate are averaged. The analysis contains the 33 model-size and peak-rate families listed in
\Cref{tab:corpus}. Time in the fitted laws is
measured in thousands of optimizer updates, with $w_0=0.032$. A run is \emph{eligible} at horizon $T$ if it completed without diverging
and $W<T$; fitting, warmup selection, and all figures use only eligible
runs.

\paragraph{Data and training.}
We follow the standard Llama-on-C4 pretraining setup used in GaLore
\citep{zhao2024galore}.\footnote{Reference implementation: \url{https://github.com/jiaweizzhao/GaLore}.}
Models are pretrained on C4 English with the \texttt{t5-base} SentencePiece
tokenizer \citep{raffel2020t5}. Training examples are randomly sampled packed
sequences, and validation is next-token cross-entropy on a fixed held-out
set. The base recipe and model configurations are summarized below; additional
experiments inherit these settings unless stated otherwise.

\begin{table}[!htbp]
\centering
\caption{\textbf{Base training configuration.}}
\label{tab:training-recipe}
\small
\begin{tabular}{lp{0.58\textwidth}}
Setting & Value\\\midrule
Training corpus & C4 English; 1,024 shards; 184,565,529,484 tokens\\
Tokenizer / storage & \texttt{t5-base} SentencePiece / \texttt{uint16}\\
Sequence length & 256\\
Batch size & 512, 131,072 tokens per update\\
Training budget & 128,000 updates; 16,777,216,000 tokens\\
AdamW $(\beta_1,\beta_2)$ / $\epsilon$ & $(0.9,0.99)$ / $10^{-6}$\\
Weight decay / global gradient clip & 0 / 1.0\\
Learning-rate schedule & Linear warmup, then constant peak rate\\
Precision  & bfloat16 \\
\bottomrule
\end{tabular}
\end{table}

\begin{table}[!htbp]
\centering
\setlength{\tabcolsep}{15pt}
\caption{\textbf{Model and execution settings.} All models use Llama decoder blocks with SwiGLU, rotary position embeddings, RMSNorm with $\epsilon=10^{-6}$, initialization standard deviation 0.02, and untied input embeddings and output heads. The configured maximum sequence length is 1,024; training uses 256 tokens.}
\label{tab:architectures}
\small
\begin{tabular}{@{}lrrr@{}}
\toprule
Setting & 60M & 100M & 350M\\
\midrule
Layers & 8 & 12 & 24\\
Hidden dimension & 512 & 640 & 1,024\\
MLP intermediate dimension & 1,376 & 1,708 & 2,736\\
Attention heads / head dimension & 8 / 64 & 10 / 64 & 16 / 64\\
Vocabulary size & 32,000 & 32,100 & 32,000\\
Total parameters & 58,073,600 & 100,117,120 & 367,969,280\\
Non-embedding parameters & 25,305,600 & 59,029,120 & 302,433,280\\
Embedding and output parameters & 32,768,000 & 41,088,000 & 65,536,000\\
Microbatch / accumulation & 512 / 1 & 128 / 4 & 64 / 8\\
Validation tokens & 10,092,544 & 10,027,008 & 10,010,624\\
\bottomrule
\end{tabular}
\end{table}

\paragraph{Batch-size ablation.} The batch-size comparison in
\Cref{fig:stability-requirements}b trains the 100M model at peak learning rates
$0.5$, $2$, and $5\times10^{-3}$ with batches of 256 and 1{,}024 sequences,
alongside the base batch of 512. All other settings match
\Cref{tab:training-recipe,tab:architectures}, including the 128{,}000-update
budget; the peak learning rate is not rescaled with batch size. Tokens per
update are therefore 65{,}536, 131{,}072, and 262{,}144, respectively.

\begin{table}[!htbp]
\centering
\setlength{\tabcolsep}{10pt}
\caption{\textbf{Experimental grid.} Peak learning rates and warmup ranges for the 33 model-size and peak-rate families; warmup grids vary with model size and peak learning rate.}
\label{tab:corpus}
\small
\begin{tabular}{@{}llcc@{}}
\toprule
Scale & Peak LRs ($10^{-3}$) & Families & Warmup range (updates)\\
\midrule
60M  & 0.1, 0.3, 0.5, 0.7, 1, 2, 3, 4, 5, 6, 7, 8 & 12 & 0--50,000\\
100M & 0.1, 0.3, 0.5, 0.7, 1, 2, 3, 4, 5, 6, 7, 8 & 12 & 0--50,000\\
350M & 0.5, 0.7, 1, 3, 4, 5, 6, 7, 8      & 9  & 0--64,000\\
\bottomrule
\end{tabular}
\end{table}

\subsection{Fitting and baselines}
\label{app:calibration}
The absolute and difference fits estimate
$(L_\infty,A,K,p,q,s)$ and $(A,C,p,q,s)$ independently. For the
difference fit, the shortest successful warmup is the reference and only
shared checkpoints are used. Positive parameters are optimized in log space
with a bounded Huber objective and deterministic multistart initialization.
Held-out fits use only their fitting data.

\begin{table}[!htbp]
\centering
\setlength{\tabcolsep}{15pt}
\caption{\textbf{Fitting settings.} Time is measured in thousands of updates. Widening the exponent bounds to $[0.05, 2]$, which allows $\beta\le 0$, gave slightly higher selection regret.}
\label{tab:fitting-settings}
\small
\begin{tabular}{ll}
\toprule
Setting & Choice\\
\midrule
Fitting optimizer & Trust-region reflective (SciPy \texttt{least\_squares})\\
Huber transition & $0.02$ loss units\\
Warmup offset & $w_0=0.032$\\
Exponent bounds & $p,q,s\in[0.05,1]$\\
Absolute amplitudes & $A,\ K/w_0^{s} \in [10^{-9}, 10^{5}]$\\
Difference amplitude & $C\in[10^{-9},10^4]$\\
Extrapolation fit window & checkpoints through 32k updates\\
Free-$s$ multistart & 35 deterministic starts\\
\bottomrule
\end{tabular}
\end{table}

\paragraph{Direct warmup-scaling baseline.}
The two-parameter baseline in \cref{app:direct-duration-scaling} fits a power
law directly to the best observed warmup at the fit horizons. It is
included to test whether extrapolating the observed duration trend is already
sufficient without modeling the loss surface.

\paragraph{Tuned rules.} The fixed-duration rule selects the warmup with the lowest loss at the last fit checkpoint, among all eligible warmups on the full grid or among the three fitting runs, and keeps it for every later horizon. The fixed-fraction rule keeps that warmup's fraction of the fit horizon and scales it with $T$. Both are mapped to the nearest eligible candidate, like the 1k and 10\% targets.

\subsection{Candidate selection and metrics}
\label{app:scoring}
All methods choose from the same successful warmup candidates at each target
horizon. Loss-law fits select the candidate with lowest predicted loss; the
direct scaling, 1k, and 10\% baselines are mapped to the nearest available
candidate. Regret is measured relative to the best observed candidate and is
averaged equally across target horizons and then across families. RMSE and
$R^2$ are computed from unmodified validation loss. In aggregate tables,
$\pm$ denotes the sample standard deviation across horizon-wise family means.

\makeatletter
\@ifundefined{assumption}{\newtheorem{assumption}{Assumption}}{}
\@ifundefined{corollary}{\newtheorem{corollary}{Corollary}}{}
\makeatother
\crefname{assumption}{Assumption}{Assumptions}
\Crefname{assumption}{Assumption}{Assumptions}
\crefformat{assumption}{Assumption~#2#1#3}
\Crefformat{assumption}{Assumption~#2#1#3}

\section{Theory supporting the main text}
\label{app:core}\label{app:det-effective-rates}
We follow \cref{sec:quadratic,sec:model}: first the quadratic warmup
tradeoff and warmup above the stability limit, then the passage
from many modes to the fitted loss law, and finally the fitted law's optimal
duration. 
\Cref{app:det-stochastic-complement} gives a complementary  discrete
result: sampling breaks order invariance and can change the loss-optimal
warmup with horizon, even with fixed population curvature. Unless otherwise
specified, the peak rate $\eta$ is fixed when $W$ and $T$ vary.

\subsection{Quadratic warmup tradeoff}
\label{app:det-quadratic}
On a quadratic loss the effect of warmup can be computed one eigendirection
(\emph{mode}) at a time. For a mode with curvature $h$ and relative peak rate
$\alpha=\eta h$, \cref{fig:det-regimes} compares the log-contraction of one
step at the peak rate, $\rho(\alpha)$, with that of one warmup step on
average, $g(\alpha)$. Warmup helps the mode when $g>\rho$, which splits the
peak rates into four ranges:
\begin{itemize}
 \item $\alpha<\alpha_{\rm cross}\simeq1.28$: warmup slows the mode
 ($g<\rho$);
 \item $\alpha_{\rm cross}<\alpha<2$: warmup speeds it up ($g>\rho>0$);
 \item $2<\alpha<\alpha_{\max}\simeq4.59$: the peak rate makes the mode grow
 ($\rho<0$), but warmup still shrinks it on average ($g>0$);
 \item $\alpha>\alpha_{\max}$: warmup makes the mode grow too ($g<0$).
\end{itemize}

\begin{figure}[!htbp]
 \centering
 \includegraphics[width=0.9\textwidth]{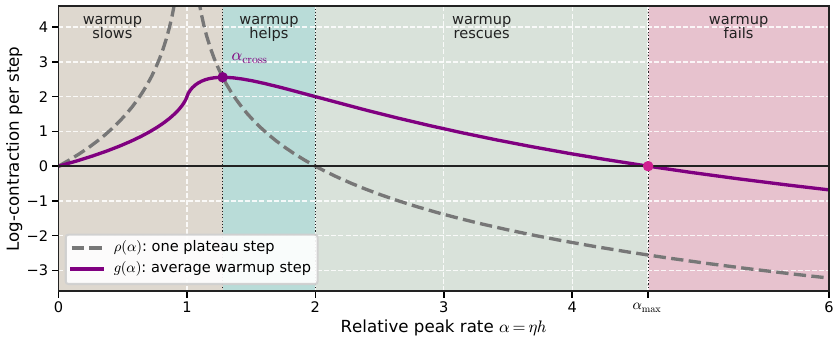}
 \caption{What warmup does to one mode. $\rho(\alpha)$: log-contraction of
 one step at the peak rate. $g(\alpha)$: its average over a linear warmup from
 $0$ to $\alpha$ (\cref{eq:det-warmup-integral}). Dots: $\alpha_{\rm cross}$,
 where $g=\rho$, and $\alpha_{\max}$, where $g=0$.}
 \label{fig:det-regimes}
\end{figure}

\paragraph{Exact modal loss.} Let $e=\theta-\theta_\star$ and
$f(\theta)-f_\star=\tfrac12 e^\top H e$ with $H\succeq 0$, let $h_i>0$ be the
curvatures of the modes that carry loss, and let $A_i$ be their initial
losses. A gradient step at rate $\eta_t$ multiplies the loss of mode $i$ by
$(1-\eta_th_i)^2$. With the linear warmup $\eta_t=\eta\min\{t/W,1\}$ ($W=0$
is the constant schedule) and $\alpha_i=\eta h_i$,
\begin{equation}
 f(\theta_T)-f_\star=\sum_i A_i\,M_{T,W}(\alpha_i),\qquad
 M_{T,W}(\alpha)=\prod_{t=1}^{W}\Bigl|1-\frac{\alpha t}{W}\Bigr|^2\,
 |1-\alpha|^{2(T-W)} .
 \label{eq:det-exact-product}
\end{equation}
At the peak rate a mode shrinks exactly when $0<\alpha<2$; we call
$\alpha=2$ the stability limit.

\paragraph{Contraction rates.}
One step at relative rate $\alpha$ multiplies a mode's loss by
$e^{-\rho(\alpha)}$. The warmup factor in \cref{eq:det-exact-product} runs
over $W$ equally spaced rates in $(0,\alpha]$; replacing the sum of their
logarithms by $W$ times its average gives
\begin{equation}
 \rho(\alpha)=-2\log|1-\alpha|,\qquad
 g(\alpha)=\int_0^1\rho(\alpha u)\,\mathrm du
 =2\Bigl[1+\frac{1-\alpha}{\alpha}\log|1-\alpha|\Bigr],
 \label{eq:det-warmup-integral}
\end{equation}
where the closed form uses the antiderivative $w\log|w|-w$ of $\log|w|$
($g$ is finite for every $\alpha>0$, and $g(1)=2$), and the continuous-warmup
approximation
\begin{equation}
 \overline M_{T,W}(\alpha)=\exp\{-g(\alpha)W-\rho(\alpha)(T-W)\}.
 \label{eq:det-envelope-app}
\end{equation}
The results below use the continuous-warmup approximation; exact integer
warmup schedules can exhibit isolated cancellations when a step lands at $\eta_th=1$.

\paragraph{When warmup helps a mode.}
Turning a plateau step into a warmup step changes a mode's log-loss by
$\rho-g$. Two identities locate the sign change:
\begin{equation}
 g(\alpha)-\rho(\alpha)=\frac2\alpha\{\alpha+\log(\alpha-1)\}\quad(\alpha>1),
 \qquad
 g'(\alpha)=\frac{\rho(\alpha)-g(\alpha)}{\alpha},
 \label{eq:det-phase-identities}
\end{equation}
the first from \cref{eq:det-warmup-integral} and the second by differentiating
$g(\alpha)=\alpha^{-1}\int_0^\alpha\rho(v)\,\mathrm dv$. For $\alpha<1$, $\rho$
increases, so its average $g$ is smaller. For $\alpha>1$, $g-\rho$ changes
sign once, at the solution $\alpha_{\rm cross}\simeq1.27846$ of
$\alpha+\log(\alpha-1)=0$. By the second identity, $g$ increases where
$g<\rho$ and decreases where $g>\rho$.

\paragraph{Two modes.}
Take a low-curvature mode $s$ with $\eta h_s<1$ and a mode $e$ near the
stability limit, $\alpha_{\rm cross}<\eta h_e<2$. With $\rho_j=\rho(\eta h_j)$
and $g_j=g(\eta h_j)$, the approximate loss is
\[
 L_{\rm quad}(W,T)=A_s\,e^{-\rho_sT+d_sW}+A_e\,e^{-\rho_eT-a_eW},
 \qquad d_s=\rho_s-g_s>0,\quad a_e=g_e-\rho_e>0,
\]
with $W$ treated as real. Solving $\partial_WL_{\rm quad}=0$ gives a line in
$T$,
\[
 W=x_\eta T+c_\eta,\qquad
 x_\eta=\frac{\rho_s-\rho_e}{d_s+a_e},\qquad
 c_\eta=\frac{1}{d_s+a_e}\log\frac{A_ea_e}{A_sd_s},
\]
whose slope changes sign where both modes contract equally fast at the peak
rate, $\eta_c=2/(h_s+h_e)$.

\begin{figure}[!htbp]
 \centering
 \includegraphics[width=\textwidth]{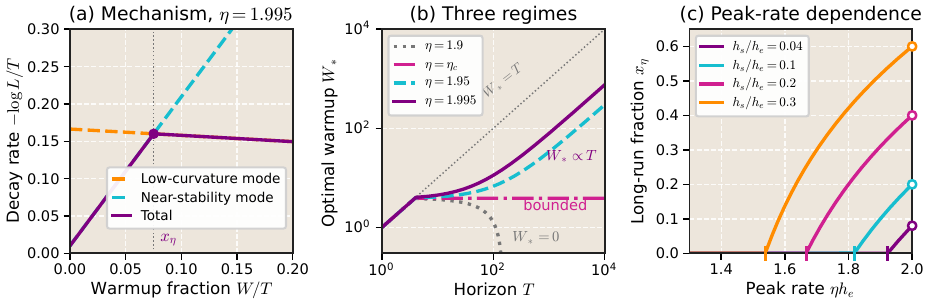}
 \caption{The two-mode optimum (\cref{thm:det-two-mode}) with the curvatures
 of \cref{fig:det-competition}, $h_e=1$ and $h_s/h_e=0.04$, so
 $\eta_c\simeq1.923$. \textbf{(a)}~At $\eta=1.995$ and $W=xT$, the loss decays
 at rate $\min(\rho_s-d_sx,\ \rho_e+a_ex)$ (solid), which is largest where the
 two modal rates (dashed) are equal, at $x=x_\eta$. \textbf{(b)}~$W^*_{\rm quad}$
 against $T$ ($A_e/A_s=100$); dotted: $W^*=T$. \textbf{(c)}~$x_\eta$ against
 the peak rate. Ticks: $\eta_c$. Circles: $2h_s/h_e$.}
 \label{fig:det-two-mode}
\end{figure}

\begin{theorem}[Two-mode optimum and peak-rate transition]
\label{prop:det-two-mode}\label{prop:det-quadratic-lr}
\label{thm:det-two-mode}
Let $A_s,A_e,d_s,a_e>0$. For every $T>0$, the unique minimizer of
$L_{\rm quad}(\cdot,T)$ on $[0,T]$ is
\begin{equation}
 W^*_{\rm quad}=\Pi_{[0,T]}\bigl[x_\eta T+c_\eta\bigr],
 \label{eq:det-two-mode-optimum}
\end{equation}
where $\Pi_{[0,T]}$ clips to $[0,T]$. If moreover $h_s<h_e/2$ and
$\alpha_{\rm cross}<\eta h_e<2$, then for all large $T$
\[
 W^*_{\rm quad}=
 \begin{cases}
  0, & \eta<\eta_c,\\
  \max(c_\eta,0), & \eta=\eta_c,\\
  x_\eta T+c_\eta, & \eta>\eta_c.
 \end{cases}
\]
For $\eta>\eta_c$,
\begin{equation}
 \frac{W^*_{\rm quad}}{T}\longrightarrow x_\eta\in(0,1),\qquad
 \frac{\mathrm d x_\eta}{\mathrm d\eta}>0,\qquad
 x_\eta\xrightarrow{\ \eta\downarrow\eta_c\ }0,\qquad
 x_\eta\xrightarrow{\ \eta\uparrow2/h_e\ }\frac{2h_s}{h_e}.
 \label{eq:main-quadratic-lr}
\end{equation}
\end{theorem}

\begin{proof}[Proof of \cref{thm:det-two-mode}]
$\partial_W^2L_{\rm quad}>0$, so the minimizer is the stationary point
clipped to $[0,T]$. On the stated range, $\eta h_s<2h_s/h_e<1$ and
$\eta h_e>\alpha_{\rm cross}$, so $d_s,a_e>0$. Since
$\rho(\alpha)=-2\log|1-\alpha|$,
\[
 \rho_s>\rho_e\iff 1-\eta h_s<\eta h_e-1\iff \eta>\eta_c ,
\]
so $x_\eta$ has the sign of $\eta-\eta_c$. Clipping $x_\eta T+c_\eta$ gives
the three cases, using $0<x_\eta<1$ for $\eta>\eta_c$.

For the properties of $x_\eta$, write $x_\eta=N/(N+G)$ with
$N=\rho_s-\rho_e>0$ and $G=g_e-g_s$. As $\eta$ grows, $\rho_s,g_s$ increase
and $\rho_e,g_e$ decrease (\cref{eq:det-phase-identities}), so $N'>0$,
$G'<0$, and $G>g(2)-g(2h_s/h_e)>0$. Hence $0<x_\eta<1$ and
$x_\eta'=(N'G-NG')/(N+G)^2>0$. As $\eta\downarrow\eta_c$, $N\to0$. As
$\eta\uparrow2/h_e$, $\rho_e\to0$ and $g_e\to2$, and
\cref{eq:det-warmup-integral} gives $x_\eta\to2h_s/h_e$.
\end{proof}

\subsection{Warmup above the stability limit}
\label{app:det-recovery}
Above the stability limit, $\alpha=\eta h_e>2$, a mode grows at every
peak-rate step, but warmup still passes through rates at which it shrinks.
Write $\rho=\rho(\alpha)$, $g=g(\alpha)$ and $a=g-\rho$. By
\cref{eq:det-warmup-integral,eq:det-phase-identities},
\[
 \rho=-2\log(\alpha-1)<0,\qquad
 a=\frac{2}{\alpha}\{\alpha+\log(\alpha-1)\}>0 .
\]
In the continuous-warmup approximation of \cref{eq:det-envelope-app}, the
mode's loss after $T$ steps is $A_ee^{-C}$, with accumulated contraction
\[
 C=gW+\rho(T-W)=\rho T+aW .
\]
Warmup prevents net amplification by horizon $T$ exactly when $C\ge0$.

\begin{figure}[!htbp]
 \centering
 \includegraphics[width=0.7\textwidth]{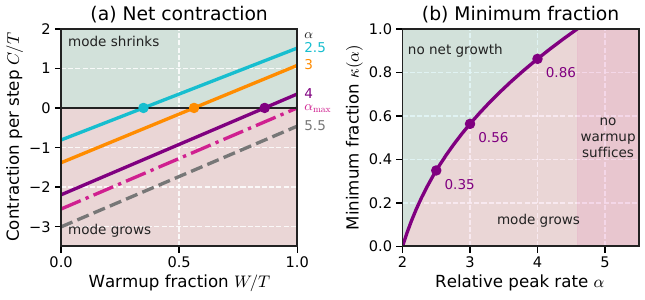}
 \caption{Warmup above the stability limit (\cref{thm:det-recovery}).
 \textbf{(a)}~$C/T=\rho+a\,W/T$ against the warmup fraction; dots mark
 $\kappa(\alpha)$. \textbf{(b)}~The threshold $\kappa(\alpha)$.}
 \label{fig:det-recovery}
\end{figure}

\begin{theorem}[Warmup required to avoid net growth]
\label{prop:det-unrecoverable}\label{thm:det-recovery}
Let $A_e>0$ and $\alpha>2$. On $(2,\infty)$, $g$ has a unique zero
$\alpha_{\max}$, given by
\begin{equation}
 (\alpha_{\max}-1)\log(\alpha_{\max}-1)=\alpha_{\max},
 \qquad \alpha_{\max}\simeq4.59112148.
 \label{eq:det-unrecoverable-cutoff}
\end{equation}
\begin{enumerate}
 \item[(a)] If $2<\alpha<\alpha_{\max}$, warmup prevents net amplification by horizon $T$ if and
 only if
 \begin{equation}
  \frac WT\ge\kappa(\alpha)
  =\frac{-\rho}{g-\rho}
  =\frac{\alpha\log(\alpha-1)}{\alpha+\log(\alpha-1)},
  \label{eq:det-rescue-fraction}
 \end{equation}
 and $\kappa$ increases from $\kappa(2)=0$ to $\kappa(\alpha_{\max})=1$.
 \item[(b)] If $\alpha\ge\alpha_{\max}$, then $C<0$ for every $W<T$. At
 $\alpha=\alpha_{\max}$, $W=T$ gives $C=0$. For $\alpha>\alpha_{\max}$,
 \begin{equation}
  \min_{0\le W\le T} A_e e^{-C}=A_e e^{-gT}\longrightarrow\infty
  \qquad(T\to\infty).
  \label{eq:det-unrecoverable-growth}
 \end{equation}
\end{enumerate}
\end{theorem}

\begin{proof}[Proof of \cref{thm:det-recovery}]
Since $a>0$, $C=\rho T+aW$ increases with $W$, so
\[
 \max_{0\le W\le T}C=gT .
\]
By \cref{eq:det-phase-identities,eq:det-warmup-integral},
\[
 g'=\frac{\rho-g}{\alpha}=-\frac a\alpha<0,\qquad
 g=2-2\,\frac{\alpha-1}{\alpha}\log(\alpha-1),
\]
so $g$ decreases from $g(2)=2$ to $-\infty$ and has one zero, given by
\cref{eq:det-unrecoverable-cutoff}.

\emph{(a)} With $\ell=\log(\alpha-1)$, so that $-\rho=2\ell$ and
$a=2+2\ell/\alpha$,
\[
 C\ge0\iff\frac WT\ge\frac{-\rho}{a}=\frac{\alpha\ell}{\alpha+\ell}=\kappa .
\]
Moreover $\kappa<1\iff g>0$; $\kappa(2)=0$ since $\ell=0$; and
$\kappa(\alpha_{\max})=1$ since $g=0$ gives $a=-\rho$. Finally,
\begin{equation}
 \kappa'(\alpha)=\frac{\ell^2+\alpha^2/(\alpha-1)}{(\alpha+\ell)^2}>0 .
 \label{eq:det-kappa-derivative}
\end{equation}

\emph{(b)} If $g=0$, then $C=\rho(T-W)\le0$, with equality only at $W=T$.
If $g<0$, then $C\le gT<0$ for every $W$.
\end{proof}
This is a finite-horizon statement: for any fixed $W$, an unstable
constant-rate phase eventually dominates as $T\to\infty$.

\subsection{From many modes to the fitted loss law}
\label{app:det-progress}\label{app:det-spectrum}
\label{app:det-exponent}\label{app:det-compression}
A single mode decays exponentially; power laws arise when error is spread
over many modes. If components with speed $z$ decay as $e^{-zC}$ and their
initial error has density $cz^{\gamma-1}$ near zero, then for large $C$
\[
 \int_0^\infty c\,z^{\gamma-1}e^{-zC}\,\mathrm dz=c\,\Gamma(\gamma)\,C^{-\gamma}:
\]
components faster than $1/C$ are gone, and what remains is the error at
speeds below $1/C$. We apply this to low-curvature modes, which give the
progress term, and to an effective persistent contribution, and then
connect the result to the fitted law by a local product approximation.

\paragraph{Progress term.}
For a low-curvature mode, $(1-\eta_th)^2\approx e^{-2\eta_th}$, so its loss is
about $e^{-2hS_T}$ with $S_T=\sum_t\eta_t$. For the linear warmup,
\begin{equation}
 S_T=\eta\,(T-W/2+1/2)\quad(W\ge1),\qquad S_T=\eta T\quad(W=0),
 \label{eq:det-progress-clock-app}
\end{equation}
so $S_T\approx\eta\tau$ with the effective training amount $\tau=T-W/2$: a
warmup step counts as half a peak-rate step. We describe the initial loss of
these modes by a measure $\mu_s$ on curvatures with a power-law density near
zero, as is standard in analyses of power-law loss curves
\citep{bordelon2024dynamical,velikanov2024tight}.

\begin{proposition}[Power-law decay from low curvatures]
\label{prop:det-progress}
Fix $\eta>0$ and $0<h_{\max}<1/\eta$. Let $\mu_s$ be a finite measure on
$(0,h_{\max}]$ with density $\psi(h)\sim c_sh^{p-1}$ near zero, $p,c_s>0$.
Then, for integers $T\to\infty$ and $0\le W\le T$,
\begin{equation}
 P(W,T)=\int M_{T,W}(\eta h)\,\mathrm d\mu_s(h)
 \sim c_s\Gamma(p)\,(2\eta)^{-p}\,\tau^{-p}.
 \label{eq:det-progress-integral}
\end{equation}
\end{proposition}
\begin{proof}
Let $\bar\alpha=\eta h_{\max}<1$. For $0\le x\le\bar\alpha$,
\[
 x\;\le\;-\log(1-x)\;\le\;x+\frac{x^2}{2(1-\bar\alpha)} .
\]
Applying this to every factor of $M_{T,W}$, with $x=\eta_th$, and using
$\sum_t\eta_t^2\le\eta S_T$ places each mode between two exponentials:
\[
 e^{-2S_T\,h(1+ch)}\;\le\;M_{T,W}(\eta h)\;\le\;e^{-2S_T\,h},
 \qquad c=\frac{\eta}{2(1-\bar\alpha)} .
\]
Integrated against $\mu_s$, both bounds are Laplace integrals with large
parameter $2S_T\ge\eta T$. By Watson's lemma \citep{olver1974asymptotics},
a density $\psi(h)\sim c_sh^{p-1}$ near zero gives
\[
 \int e^{-2S_T\,h}\,\mathrm d\mu_s(h)\;\sim\;c_s\Gamma(p)\,(2S_T)^{-p},
\]
and the same holds for the lower bound after the change of variable
$u=h(1+ch)$, which leaves the density near zero unchanged. Finally,
$S_T=\eta(\tau+1/2)\sim\eta\tau$ by \cref{eq:det-progress-clock-app}, which
gives \cref{eq:det-progress-integral}.
\end{proof}
The constant is absorbed into the amplitude $A_\eta$. The asymptotic power
law requires suitable loss-weighted mass at arbitrarily small curvatures;
finitely many modes can approximate it over an intermediate range of $T$.

\paragraph{Persistent term.}
For a reference direction with $\rho_\eta=\rho(\eta h_e)$,
$g_\eta=g(\eta h_e)$, $a_\eta=g_\eta-\rho_\eta>0$ and $g_\eta>0$, that is,
$\alpha_{\rm cross}<\eta h_e<\alpha_{\max}$, the accumulated contraction of
\cref{app:det-recovery} can be written with $T=\tau+W/2$ as
\begin{equation}
 C_\eta=\rho_\eta T+a_\eta W=\rho_\eta\tau+b_\eta W,\qquad
 b_\eta=g_\eta-\rho_\eta/2>0 .
 \label{eq:det-effective-clock-app}
\end{equation}
($b_\eta\ge g_\eta$ if $\rho_\eta\le0$, and $b_\eta=a_\eta+\rho_\eta/2$ if
$\rho_\eta>0$.) We model the aggregate persistent error as a mixture of decay rates
sharing the accumulated-contraction variable $C_\eta$.

\begin{assumption}[Effective persistent error]
\label{ass:det-relaxation}
At each peak rate, the persistent contribution is
\[
 E_\eta(W,T)=\int_{(0,\infty)}e^{-zC_\eta(W,T)}\,\mathrm d\mu_\eta(z),
\]
where $\mu_\eta$ is a finite measure with density
$\psi_\eta(z)\sim c_\eta z^{\gamma-1}$ near zero, $c_\eta,\gamma>0$, and
$\gamma$ is the same for the peak rates being compared.
\end{assumption}

Here $z$ is an effective decay rate, not a Hessian eigenvalue.

\begin{proposition}[Power-law persistent error]
\label{prop:det-relaxation}
Under \cref{ass:det-relaxation}, as $C_\eta\to\infty$,
\begin{equation}
 E_\eta\sim B_\eta C_\eta^{-\gamma},\qquad B_\eta=c_\eta\Gamma(\gamma).
 \label{eq:det-parent-exact}
\end{equation}
\end{proposition}
\begin{proof}
This is Watson's lemma \citep{olver1974asymptotics} with large
parameter $C_\eta$, as in the proof of \cref{prop:det-progress}.
\end{proof}

\paragraph{Local product approximation.}
\Cref{prop:det-relaxation} motivates the power-law surrogate
\[
 \bar E_\eta(\tau,W)=B_\eta(\rho_\eta\tau+b_\eta W)^{-\gamma},
\]
which the fitted law replaces by a product of powers $\tau^{-q}W^{-s}$. The
surrogate is the large-$C_\eta$ form of $E_\eta$, not $E_\eta$ itself: at finite
$C_\eta$ the local exponents of $E_\eta$ sum to
$-C_\eta E_\eta'(C_\eta)/E_\eta(C_\eta)$, which tends to $\gamma$ only as
$C_\eta\to\infty$. We assume $\alpha_{\rm cross}<\eta h_e<2$, so that
$\rho_\eta,a_\eta,b_\eta>0$; above the stability limit, $\rho_\eta<0$ would give
a negative time exponent, so the construction does not cover that regime.

\begin{theorem}[Product approximation]
\label{thm:det-product}\label{lem:det-compression}
Fix a reference point $\tau_0,W_0>0$, let $D_{\eta,0}=\rho_\eta\tau_0+b_\eta W_0$,
and let
\begin{equation}
 q_\eta=\frac{\gamma\rho_\eta\tau_0}{D_{\eta,0}},\qquad
 s_\eta=\frac{\gamma b_\eta W_0}{D_{\eta,0}},\qquad
 q_\eta+s_\eta=\gamma.
 \label{eq:det-local-powers-app}
\end{equation}
The product
\begin{equation}
 E_{{\rm loc},\eta}
 =\bar E_\eta(\tau_0,W_0)\Bigl(\frac{\tau}{\tau_0}\Bigr)^{-q_\eta}
 \Bigl(\frac{W}{W_0}\Bigr)^{-s_\eta}
 =K_\eta\,\tau^{-q_\eta}W^{-s_\eta}
 \label{eq:det-local-product}
\end{equation}
is the only product of powers that matches $\bar E_\eta$ and its derivatives in
$\log\tau$ and $\log W$ at $(\tau_0,W_0)$. For all $\tau,W>0$,
\begin{equation}
 0\le\log\frac{E_{{\rm loc},\eta}}{\bar E_\eta}
 \le\frac{\gamma}{8}
 \left[\log\frac{\tau/W}{\tau_0/W_0}\right]^2 .
 \label{eq:det-product-error}
\end{equation}
In particular, it matches the effect of lengthening warmup at a fixed
budget $T=\tau+W/2$:
\begin{equation}
 \left.\frac{\partial\log\bar E_\eta}{\partial W}\right|_{T,\,(\tau_0,W_0)}
 =-\frac{\gamma a_\eta}{D_{\eta,0}}
 =\frac{q_\eta}{2\tau_0}-\frac{s_\eta}{W_0}.
 \label{eq:det-parent-warmup-match}
\end{equation}
\end{theorem}
The exponents split $\gamma$ according to the share of each term in
$\rho_\eta\tau+b_\eta W$ at the reference point, and the error vanishes at a
fixed warmup fraction.

\begin{proof}
Differentiating $\log\bar E_\eta$ in $\log\tau$ and $\log W$ gives $-q_\eta$ and
$-s_\eta$ at the reference point, and matching the value fixes $K_\eta$. With
$\theta=q_\eta/\gamma$ and $v=\log\frac{\tau/W}{\tau_0/W_0}$,
\[
 \frac1\gamma\log\frac{E_{{\rm loc},\eta}}{\bar E_\eta}
 =\log\bigl(\theta e^v+1-\theta\bigr)-\theta v ,
\]
which vanishes with its first derivative at $v=0$ and has second derivative
in $(0,\tfrac14]$; Taylor's theorem gives \cref{eq:det-product-error}. The
warmup derivative follows from $\partial\tau/\partial W=-1/2$ and
$b_\eta-\rho_\eta/2=a_\eta$.
\end{proof}

\begin{corollary}[Higher peak rates shift the exponent toward warmup]
\label{prop:det-exponent-allocation}
Hold $h_e$, $\gamma$ and $W_0/\tau_0$ fixed. On
$\alpha_{\rm cross}<\alpha=\eta h_e<2$, $s_\eta$ increases and $q_\eta$
decreases with $\eta$, with $s_\eta\to\gamma$ and $q_\eta\to0$ as
$\alpha\uparrow2$.
\end{corollary}
\begin{proof}
With $b=g-\rho/2$ and \cref{eq:det-warmup-integral},
\begin{equation}
 \frac{s_\eta}{q_\eta}
 =\frac{b(\alpha)}{\rho(\alpha)}\,\frac{W_0}{\tau_0},\qquad
 \frac{b(\alpha)}{\rho(\alpha)}
 =-\frac1{\log(\alpha-1)}+\frac12-\frac1\alpha .
 \label{eq:det-allocation-ratio}
\end{equation}
Its derivative $1/\{(\alpha-1)\log^2(\alpha-1)\}+1/\alpha^2$ is positive, and it
diverges as $\alpha\uparrow2$. Since $q_\eta+s_\eta=\gamma$, the claims
follow.
\end{proof}

Adding the progress term gives the fitted law in \cref{eq:det-fitted-law}.
We use $W+w_0$ so that zero warmup is finite. The identity $q+s=\gamma$
belongs only to the local approximation, so we neither impose it nor use it to
infer the global large-$T$ behavior of the fitted law. Instead, we fit
$p,q,s$ freely at each peak learning rate. Each peak rate then has its own
law, which can be fitted from a few short runs at that rate and extrapolated
to longer horizons without runs at other peak rates. Sharing $q+s$ across peak
rates fits the observed losses nearly as well (\cref{app:shared-exponent-sum}),
but it requires fitting runs at several peak rates.

\subsection{Optimal duration under the fitted law}
\label{app:det-law-proof}\label{app:proof}
With coefficients $A,K,p,q,s,w_0>0$ and $L_\infty$ fixed, the fitted law at
horizon $T$ is
\begin{equation}
 \mathcal L_T(W)=L_\infty+A\tau^{-p}+K\tau^{-q}U^{-s},\qquad
 U=W+w_0,\quad\tau=T-W/2,\quad0\le W\le T.
 \label{eq:det-fitted-law}
\end{equation}

\begin{proposition}[Unique finite-horizon optimum]
\label{prop:det-finite-optimum}
$\mathcal L_T$ is strictly convex on $[0,T]$. Its minimizer $W_*$ is $0$ if
$\mathcal L_T'(0)\ge0$, is $T$ if $\mathcal L_T'(T)\le0$, and otherwise is the
unique solution of
\begin{equation}
 ApU_*^{s+1}=K\tau_*^{m}\Bigl(2s-\frac{qU_*}{\tau_*}\Bigr),\qquad m=p+1-q.
 \label{eq:det-interior-balance}
\end{equation}
\end{proposition}
\begin{proof}
With $E=K\tau^{-q}U^{-s}$,
\[
 \mathcal L_T'=\frac{Ap}{2}\tau^{-p-1}+E\Bigl(\frac{q}{2\tau}-\frac sU\Bigr),\qquad
 \mathcal L_T''=\frac{Ap(p+1)}{4}\tau^{-p-2}
 +E\Bigl[\Bigl(\frac{q}{2\tau}-\frac sU\Bigr)^2+\frac{q}{4\tau^2}+\frac{s}{U^2}\Bigr]>0 .
\]
Multiplying $\mathcal L_T'=0$ by $2\tau^{p+1}U^{s+1}$ gives
\cref{eq:det-interior-balance}.
\end{proof}

\paragraph{Long-horizon regimes.}
The balance in \cref{eq:det-interior-balance} weighs the cost of one more warmup
step against its lasting benefit. When $w_0\ll W_*\ll T$, we have
$U_*\approx W_*$ and $\tau_*\approx T$, and it reduces to
\[
 \underbrace{Ap\,W_*^{s+1}}_{\text{cost}}\;\approx\;\underbrace{2sK\,T^{m}}_{\text{benefit}},
 \qquad\text{so}\qquad
 W_*\approx d\,T^{\beta},
 \qquad
 d=\Bigl(\frac{2sK}{Ap}\Bigr)^{1/(s+1)},
 \qquad
 \beta=\frac{p+1-q}{s+1}.
\]
The exponent $\beta$ decides what happens as $T\to\infty$:
\begin{center}
\small
\begin{tabular}{lll}
\toprule
Exponent & Optimal warmup & Regime\\
\midrule
$\beta<0$ & $W_*=0$ & no warmup\\
$\beta=0$ & $W_*\to\max\{0,\,d-w_0\}$ & bounded warmup\\
$0<\beta<1$ & $W_*\approx d\,T^{\beta}$ & sublinear growth\\
$\beta=1$ & $W_*/T\to x_1$ & proportional growth\\
$\beta>1$ & $W_*/T\to x_\infty=\min\{1,\,2s/(q+s)\}$ & proportional growth\\
\bottomrule
\end{tabular}
\end{center}
For $\beta\le0$ the balance point does not grow with $T$, and for $0<\beta<1$
it grows but stays small relative to $T$, so the approximation above is
self-consistent. For $\beta\ge1$ it would give $W_*\gtrsim T$, so instead the
optimum settles at a fixed fraction $x=W_*/T$. Writing $W=xT$ and dropping
$w_0$, the loss above $L_\infty$ is, up to the common factor $T^{-(q+s)}$,
\[
 A\,T^{q+s-p}\,(1-x/2)^{-p}\;+\;K\,x^{-s}(1-x/2)^{-q}.
\]
For $\beta>1$, $q+s<p$, so the first term vanishes, and minimizing the second
gives $x_\infty$. For $\beta=1$, both terms remain, and
$x_1=\min\{1,\,y_*/(1+y_*/2)\}$, where $y_*>0$ solves
\begin{equation}
 Apy_*^{s+1}=K(2s-qy_*).
 \label{eq:det-critical-fraction}
\end{equation}

\paragraph{Finite horizons.} The regimes above describe $T\to\infty$; the
following result holds at every finite horizon.

\label{app:det-finite-horizon}\label{app:finite-horizon}
\begin{corollary}[Growth at finite horizons]
\label{cor:det-finite-growth}
If $q\le p+1$, the optimal warmup strictly increases with $T$ wherever it is
interior.
\end{corollary}
\begin{proof}
At an interior optimum, $\mathcal L_T'(W_*)=0$ and $\mathcal L_T''(W_*)>0$, so
$W_*'(T)=-\partial_T\mathcal L_T'/\mathcal L_T''$. Differentiating
$\mathcal L_T'$ in $T$ at fixed $W$ and using $\mathcal L_T'=0$,
\[
 \partial_T\mathcal L_T'=-\frac{Ap}{2}(p+1-q)\,\tau^{-p-2}-\frac{qE}{2\tau^2}<0 ,
\]
so $W_*'(T)>0$.
\end{proof}

\subsection{A stochastic mechanism for horizon dependence}
\label{app:det-stochastic-complement}\label{app:sampling-persistence}
The quadratic model above is deterministic. Minibatch sampling gives a second
route to horizon-dependent warmup, even with fixed curvature and a stable peak
rate. We use the exact second-moment dynamics of least squares with Gaussian
inputs \citep{bordelon2022sgd}.

\paragraph{Setup.}
Inputs are Gaussian, labels are noiseless, and each step draws a fresh
minibatch of size $B$:
\[
x\sim\mathcal N(0,H),
\qquad
y=x^\top\theta_\star .
\]
In the eigenbasis of $H$, let $\lambda_i$ be the eigenvalues and
$\Lambda=\operatorname{diag}(\lambda)$. We track the expected squared error of
each mode and the expected loss,
\[
m_{i,t}=\mathbb E\bigl[e_{i,t}^2\bigr],
\qquad
R_t=\tfrac12\sum_i\lambda_i\,m_{i,t},
\qquad
e_t=\theta_t-\theta_\star .
\]
A step with learning rate $u$ updates these moments linearly:
\begin{equation}
m_{t+1}=\mathcal A(\eta_t)\,m_t,
\qquad
\mathcal A(u)=
\underbrace{I-2u\Lambda+u^2\bigl(1+B^{-1}\bigr)\Lambda^2}_{\text{each mode on its own}}
+\underbrace{\frac{u^2}{B}\,\lambda\lambda^\top}_{\text{sampling noise}} .
\label{eq:main-moments}
\end{equation}
This follows from the Gaussian identity
\[
\mathbb E\bigl[\widehat H\Sigma\widehat H\bigr]
=\bigl(1+B^{-1}\bigr)H\Sigma H+B^{-1}H\operatorname{tr}(H\Sigma),
\]
where $\widehat H$ is the minibatch covariance. The sampling-noise term moves
error between modes. As $B\to\infty$ it vanishes, and \cref{eq:main-moments}
reduces to gradient descent. Warmup uses $\eta_t=\eta\min\{(t+1)/W,\,1\}$.

\label{app:step-ordering}Because the sampling-noise term couples modes with
different curvatures, steps at different learning rates no longer commute, so
the order of the learning rates now matters.
\newpage
\begin{figure}[!htbp]
\centering
\includegraphics[width=\textwidth]{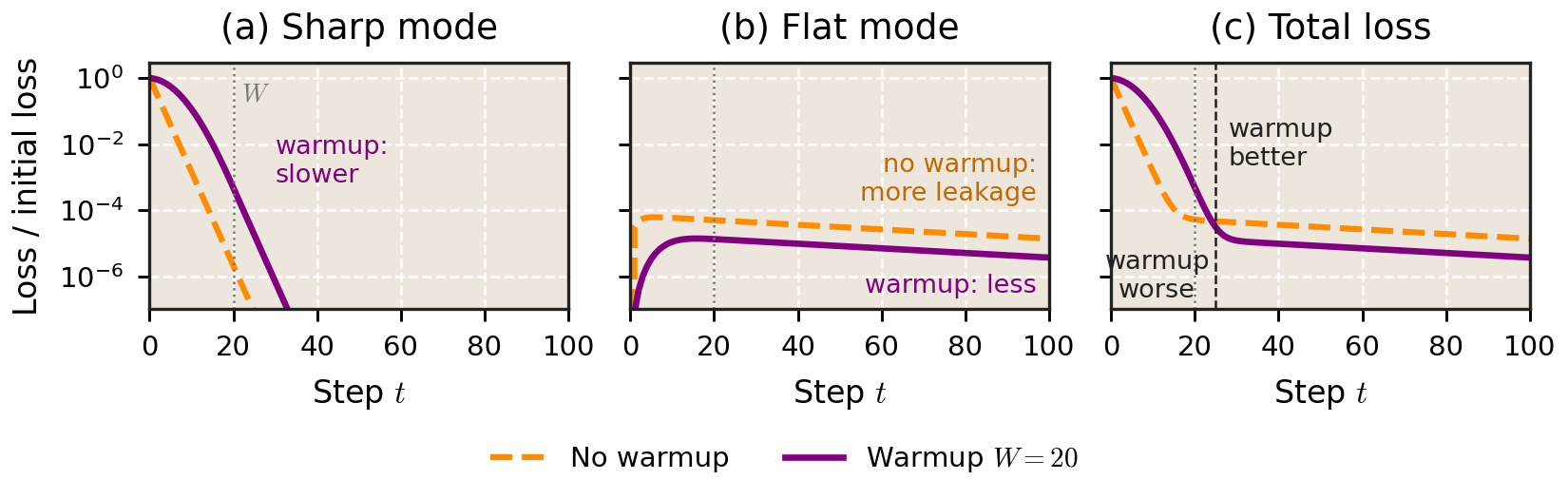}
\caption{\textbf{Sampling flips the warmup ranking.} The two-mode example of
\cref{eq:two-mode-updates} with $\epsilon=0.02$, $\eta h=0.4$, $B=2$, and
$W=20$; losses are relative to the initial loss.
\textbf{(a)} Warmup empties the sharp mode more slowly.
\textbf{(b)} Warmup leaks less error into the flat mode.
\textbf{(c)} Warmup is worse until step 25, then settles at about a quarter
of the no-warmup loss. The exact recursion in \cref{eq:main-moments} gives
the same picture, with the crossing at step 24.}
\label{fig:stochastic-reversal}
\end{figure}

\begin{proposition}[Batch size and the stable rate]
\label{prop:stochastic-ceiling}
At a constant rate $u>0$, $m_t\to0$ for every initial state if and only if
\[
\bigl(1+B^{-1}\bigr)u\lambda_i<2\ \text{ for all } i
\qquad\text{and}\qquad
\frac uB\sum_i\frac{\lambda_i}{2-(1+B^{-1})u\lambda_i}<1 .
\]
Both conditions tighten as $B$ decreases, so the supremum $u_B$ of stable rates
increases with $B$. For a single mode of curvature $h$,
\[
u_B=\frac{2B}{(B+2)\,h},
\]
and in general $u_B\to2/\max_i\lambda_i$ as $B\to\infty$.
\end{proposition}

\begin{proof}
Write $\mathcal A(u)=G+(u^2/B)\,\lambda\lambda^\top$, where $G\ge0$ is diagonal.
Such a matrix has spectral radius below one exactly when
\[
\max_i G_{ii}<1
\qquad\text{and}\qquad
\frac{u^2}{B}\,\lambda^\top(I-G)^{-1}\lambda<1 .
\]
These are the two stated conditions.
\end{proof}

\paragraph{A two-mode example.}
Consider two modes:
\begin{itemize}
\item a \emph{sharp} mode with curvature $h$, which holds all the initial error;
\item a \emph{flat} mode with curvature $\epsilon h$, $\epsilon\ll1$, which
starts with no error.
\end{itemize}
Write $a=uh$ for the step size relative to the sharp curvature, and
$c=1+2/B$. Dropping terms of order $\epsilon^2$, one step of
\cref{eq:main-moments} reads
\begin{equation}
\begin{aligned}
m_{\mathrm{sharp}} &\;\leftarrow\; \bigl(1-2a+c\,a^2\bigr)\,m_{\mathrm{sharp}},\\
m_{\mathrm{flat}} &\;\leftarrow\; \bigl(1-2\epsilon a\bigr)\,m_{\mathrm{flat}}
\;+\;\frac{\epsilon a^2}{B}\,m_{\mathrm{sharp}} .
\end{aligned}
\label{eq:two-mode-updates}
\end{equation}
The sharp mode decays quickly. The flat mode decays slowly, and sampling noise
leaks error into it from the sharp mode. Per step,
\[
\frac{\text{error leaked into the flat mode}}{\text{error removed from the sharp mode}}
=\frac{\epsilon a^2/B}{2a-c\,a^2}
=\frac{\epsilon}{B}\cdot\frac{a}{2-c\,a}.
\]
This ratio grows with $a$: large steps are leaky, and small steps are clean.

\paragraph{Why warmup loses early and wins later.}
Warmup replaces the first large steps with small ones. This has two effects:
\begin{itemize}
\item \textbf{Cost.} The sharp mode is emptied more slowly
(\cref{fig:stochastic-reversal}a). This extra error decays quickly after
warmup.
\item \textbf{Benefit.} Less error leaks into the flat mode
(\cref{fig:stochastic-reversal}b). This saving decays slowly, because the flat
mode shrinks only by a factor $1-2\epsilon a$ per step.
\end{itemize}
After warmup, both schedules take the same steps. The difference between their
losses $n$ steps after warmup is therefore a sum of two decaying terms:
\[
R(W)-R(0)
=\underbrace{C\,r_{\mathrm{fast}}^{\,n}}_{\text{cost}}
-\underbrace{D\,r_{\mathrm{slow}}^{\,n}}_{\text{benefit}},
\qquad
C,D>0,
\qquad
r_{\mathrm{fast}}<r_{\mathrm{slow}}.
\]
The cost dominates first and the benefit later, so the ranking flips once:
warmup is worse early and better late (\cref{fig:stochastic-reversal}c).
Without sampling ($B\to\infty$), nothing leaks, so $D=0$ and warmup is worse
at every horizon. The example shows that sampling alone can make the warmup
ranking depend on the horizon.

\section{Robustness of warmup behavior}
\label{app:additional-evidence}

\subsection{Warmup ordering after decay}
\label{app:decay-branching}

We branch the intermediate model's warmup--stable trajectories into linear
decay after warmup has ended. For each peak rate and horizon, we compare validation-loss rankings
immediately before decay and at its endpoint over the same completed branches.
Spearman correlation measures how well the ordering of warmup durations is preserved, regardless of the overall loss reduction from decay. We also report agreement between pairwise orderings and whether the best duration is unchanged. Post-decay regret is the endpoint loss of the warmup that was best before decay, minus the best endpoint loss among those branches.

\begin{table}[!htbp]
\centering
\setlength{\tabcolsep}{15pt}
\caption{\textbf{Decay-branching design.} Model: 100M; AdamW; seed 0. Fork horizons
$C\in\{8,16,32,64,128\}$k updates; linear decay to zero lasts 5k updates, with
endpoint $C+5$k. Only $W\le C-2$k is eligible. The comparison uses the 128 branches that completed.}
\label{tab:decay-grid}
\small
\begin{tabular}{rl}

Peak LR ($10^{-3}$) & Source warmups (updates)\\
\midrule
0.1 & 0, 256, 1,000, 4,000, 16,000\\
0.3 & 0, 128, 512, 2,000, 8,000\\
1   & 0, 512, 2,000, 8,000, 30,000\\
3   & 512, 1,500, 8,000, 16,000, 30,000\\
5   & 768, 2,000, 6,000, 16,000, 30,000\\
7   & 1,000, 3,000, 8,000, 20,000, 50,000\\
\bottomrule
\end{tabular}
\end{table}

\begin{figure}[!htbp]
\centering
\includegraphics[width=0.6\textwidth]{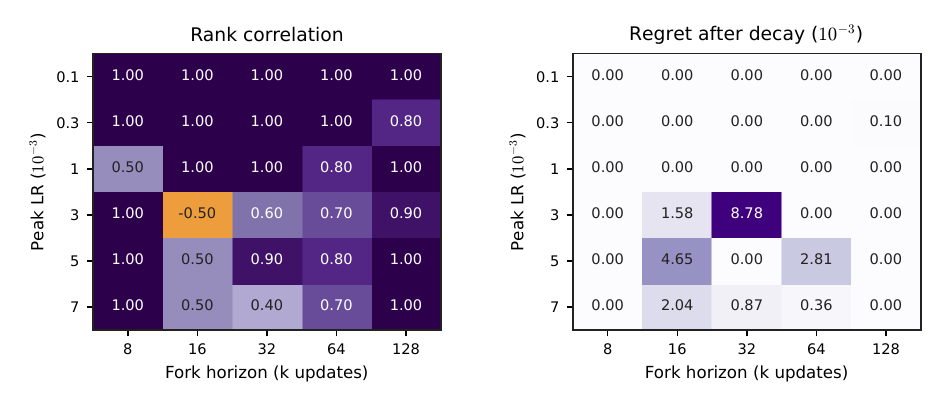}
\vspace{-1em}
\caption{\textbf{Warmup ordering before and after decay.}
Left: Spearman correlation between warmup rankings at the fork and after 5k decay
updates, within each peak-rate family ($1$: identical ordering; $-1$: reversed).
Right: post-decay regret of the warmup selected before decay, in $10^{-3}$ loss
units. Each cell compares the same warmups before and after decay.}
\label{fig:decay-ordering}
\end{figure}

\Cref{fig:decay-ordering,tab:decay-ordering} show that decay mostly
preserves the ordering of warmup durations. The best duration is unchanged in 22 of
30 fork--family cases, and in 5 of 6 families at the latest fork. Where
it changes, at intermediate forks and higher peak rates, decay favors a
longer warmup. Warmup--stable comparisons are therefore a reliable guide to warmup choice under decay.

\begin{table}[!htbp]
\centering
\setlength{\tabcolsep}{10pt}
\caption{\textbf{Ordering agreement and selection after decay.}
Preserved pairs have the same strict loss ordering before and after decay. ``Same best'' counts peak-rate families whose best warmup is unchanged by decay. Mean regret weights families equally; maximum regret is the largest family regret. Both are in $10^{-3}$ loss units.}
\label{tab:decay-ordering}
\small
\begin{tabular}{rrrrrr}

Fork (k) & Branches & Preserved pairs & Same best & Mean regret & Max regret\\
\midrule
8   & 18 & 19/20 & 6/6 & 0.00 & 0.00\\
16  & 22 & 27/31 & 3/6 & 1.38 & 4.65\\
32  & 29 & 50/56 & 4/6 & 1.61 & 8.78\\
64  & 29 & 49/56 & 4/6 & 0.53 & 2.81\\
128 & 30 & 57/60 & 5/6 & 0.02 & 0.10\\
\bottomrule
\end{tabular}
\end{table}

\subsection{Horizon and peak-rate dependence with Muon}
\label{app:muon}

We repeat the warmup sweep with Muon on the 100M parameter model
\citep{jordan2024muon,liu2025muon}. Muon updates the hidden attention and MLP
matrices; auxiliary AdamW updates the embeddings, output head, and normalization
weights. We vary only Muon's peak learning rate and warmup duration, keeping the auxiliary AdamW learning rate and warmup fixed across the sweep. Both groups keep
their own learning rate constant after warmup. \Cref{tab:muon-protocol} gives
the optimizer settings and run coverage.

\Cref{fig:muon-warmup,tab:muon-minima} show the same dependence on horizon and peak rate as with AdamW. At the lowest peak rate the best warmup stays short and the loss is nearly flat across short durations. At higher rates it grows with the horizon. The preferred durations and learning-rate scales differ from AdamW, but the warmup tradeoff carries over to Muon.

\begin{table}[!htbp]
\centering
\caption{\textbf{Muon protocol and controls.}
Model: 100M; seed 0; 128k updates. All remaining data, batch, precision,
clipping, and evaluation settings match
\cref{tab:training-recipe,tab:architectures}.}
\label{tab:muon-protocol}
\small
\begin{tabular}{lp{0.55\textwidth}}

Setting & Value\\
\midrule
Muon parameters & Hidden matrices: 59,013,120\\
Auxiliary AdamW parameters & Embeddings, output head, norms: 41,104,000\\
Muon momentum / Nesterov & 0.95 / enabled\\
Newton--Schulz iterations & 5\\
Muon update scaling &
$0.2\sqrt{\max(d_{\rm out},d_{\rm in})}$ after approximate orthogonalization\\
Muon peak LR ($10^{-3}$) & 0.1, 0.2, 0.5, 1, 2\\
Muon warmups (updates) & 0, 512, 2,000, 8,000, 16,000, 50,000\\
Fixed auxiliary peak LR / warmup & $10^{-3}$ / 2,000 updates\\
Auxiliary $(\beta_1,\beta_2)$ / $\epsilon$ & $(0.9,0.99)$ / $10^{-6}$\\
Weight decay / warmup shape & 0 for both groups / linear for each group\\
Runs & 30: 6 warmups at each of 5 Muon peak rates\\
\bottomrule
\end{tabular}
\end{table}

\begin{table}[!htbp]
\centering
\setlength{\tabcolsep}{15pt}
\caption{\textbf{Best measured Muon warmup (updates).}
Columns give the training horizon; only completed runs with $W<T$ are eligible.
Ties favor shorter warmup. The loss curves in \cref{fig:muon-warmup} show how
sharply each minimum is determined.}
\label{tab:muon-minima}
\small
\begin{tabular}{rrrrrr}

Muon peak LR ($10^{-3}$) & 8k & 16k & 32k & 64k & 128k\\
\midrule
0.1 & 0     & 512   & 512   & 512    & 2,000\\
0.2 & 512   & 2,000 & 2,000 & 8,000  & 8,000\\
0.5 & 512   & 2,000 & 8,000 & 8,000  & 8,000\\
1   & 512   & 2,000 & 8,000 & 16,000 & 16,000\\
2   & 2,000 & 2,000 & 8,000 & 16,000 & 50,000\\
\bottomrule
\end{tabular}
\end{table}

\begin{figure}[!htbp]
\centering
\includegraphics[width=\textwidth]{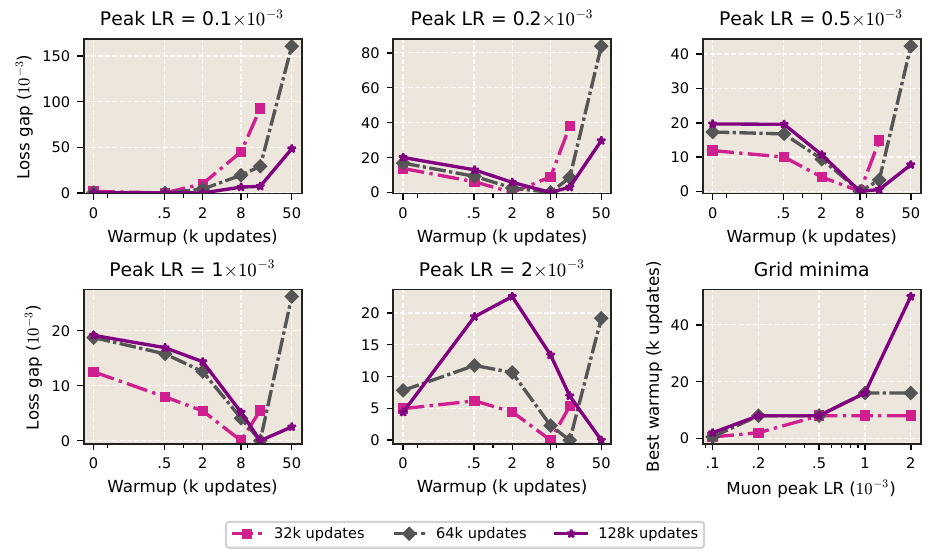}
\caption{\textbf{Warmup dependence with Muon.}
Loss gaps are relative to the best tested warmup at each horizon and Muon peak
rate, using only $W<T$ from the 30 completed runs.}
\label{fig:muon-warmup}
\end{figure}

\subsection{Horizon dependence under different warmup shapes}
\label{app:warmup-shapes}

We test whether warmup duration retains its horizon and peak-rate dependence
within each warmup shape. On the 60M parameter model, we vary duration under linear,
half-cosine, and concave-quadratic warmup, followed by a constant peak rate.
The data, initialization seed, optimizer, and batch settings match the base grid.
\Cref{tab:shape-schedules} defines the schedules and sweep;
\cref{fig:shape-schedules} shows their learning-rate profiles.

For a warmup multiplier $m$, the accumulated learning rate can be written as
\[
 \tau_m=T-c_mW,
 \qquad
 c_m=1-\int_0^1m(u)\,\mathrm du.
\]
Substituting this quantity into \cref{eq:absolute-law}, the same interior balance
as in \cref{sec:law} gives, for $w_0\ll W_*\ll T$,
\[
 W_*\approx
 \left(\frac{sK}{c_mAp}\right)^{1/(s+1)}T^\beta,
 \qquad
 \beta=\frac{p+1-q}{s+1}.
\]
Holding the other law coefficients fixed, a smaller $c_m$ therefore predicts
longer warmup with the same exponent of $T$. Linear and half-cosine warmup have
the same accumulated learning rate. Concave-quadratic warmup has a smaller
progress penalty per warmup step, which shifts the balance toward longer
durations. The persistent-error coefficients can also depend on warmup shape.

\Cref{fig:warmup-shapes,tab:shape-minima} compare durations across horizons
within each shape. The shift toward longer warmup is clearest at the highest
peak rate, while lower rates favor shorter durations. Linear and half-cosine
warmup show similar qualitative trends. At the highest peak rate, the
concave-quadratic optimum continues to move toward longer durations at the
latest horizon, consistent with its smaller progress penalty.

\begin{table}[!htbp]
\centering
\setlength{\tabcolsep}{15pt}
\caption{\textbf{Warmup shapes and design.}
Here $x=t/W$ during warmup, $\eta_t=\eta m(x)$, and $m=1$ after warmup.
The last columns use continuous integrals, with $T\ge W$.
Model: 60M; AdamW; peak LRs $\{0.5,1,5\}\times10^{-3}$;
nonlinear warmups $\{500,1000,5000,15000,50000\}$; horizon 128k.}
\label{tab:shape-schedules}
\small
\begin{tabular}{lccc}

Shape & $m(x)$ & $\int_0^1m(x)\,\mathrm dx$
& $\tau=\int_0^T\eta_t/\eta\,\mathrm dt$\\
\midrule
Linear & $x$ & $1/2$ & $T-W/2$\\
Half-cosine & $(1-\cos(\pi x))/2$ & $1/2$ & $T-W/2$\\
Concave quadratic & $2x-x^2$ & $2/3$ & $T-W/3$\\
\bottomrule
\end{tabular}
\end{table}

\begin{figure}[!htbp]
\centering
\includegraphics[width=0.50\textwidth]{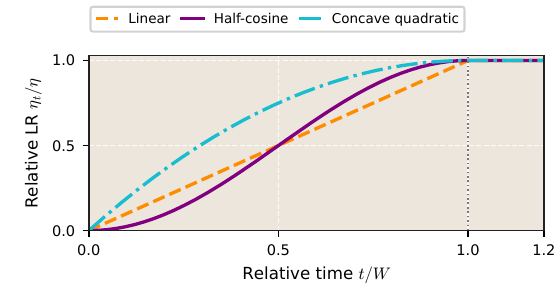}
\caption{\textbf{Warmup learning-rate profiles.}
Relative learning rate $\eta_t/\eta$ against relative time $t/W$.
All shapes reach the same peak at the end of warmup and remain constant
afterward.}
\label{fig:shape-schedules}
\end{figure}

\begin{figure}[!htbp]
\centering
\includegraphics[width=0.9\textwidth]{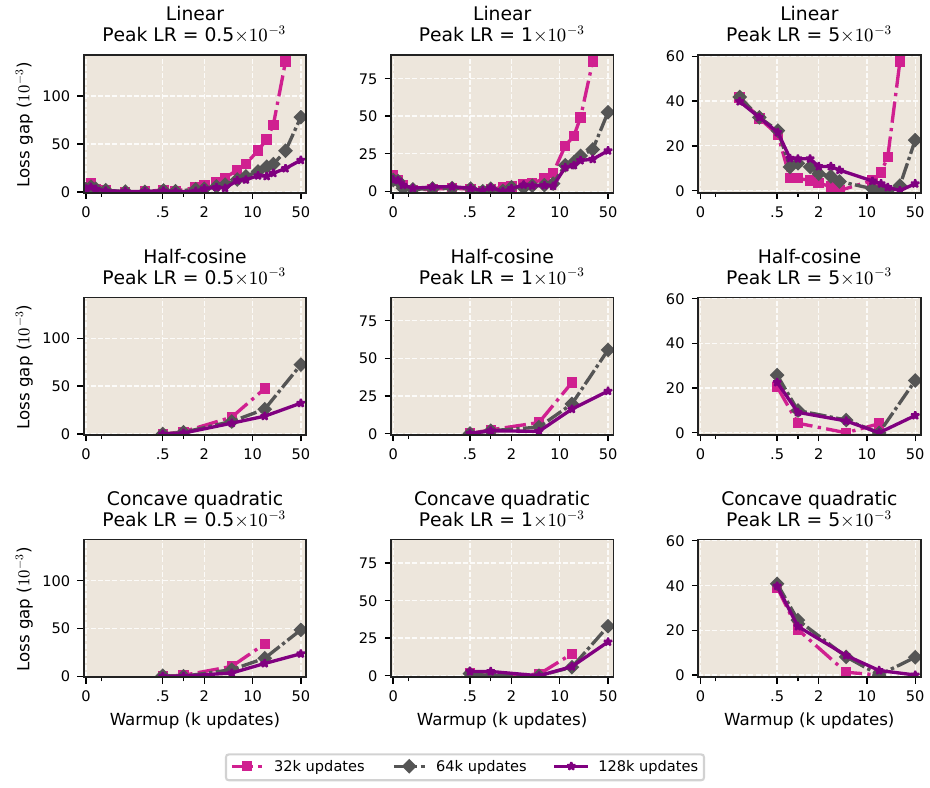}
\caption{\textbf{Horizon-dependent warmup under different shapes.}
Each panel compares durations within a fixed shape and peak rate. Rows show
linear, half-cosine, and concave-quadratic warmup; columns show peak learning
rates. Loss gaps are relative to the best tested duration within each shape,
peak rate, and horizon, restricted to $W<T$. Each shape uses its own tested
duration grid; linear curves follow the base-grid eligibility rule in
\cref{app:observations}. Lines connect measured candidates.}
\label{fig:warmup-shapes}
\end{figure}

\begin{table}[!htbp]
\centering
\setlength{\tabcolsep}{10pt}
\caption{\textbf{Best measured warmup across horizons within each shape (updates).}
Each row tracks the loss-minimizing duration for a fixed shape and peak rate.
Columns give the training horizon; only $W<T$ is eligible, and exact ties favor
shorter warmup.}
\label{tab:shape-minima}
\small
\begin{tabular}{lrrrrrr}

Shape & Peak LR ($10^{-3}$) & 8k & 16k & 32k & 64k & 128k\\
\midrule
Linear & 0.5 & 256 & 256 & 256 & 384 & 768\\
Linear & 1 & 128 & 128 & 768 & 768 & 1,500\\
Linear & 5 & 768 & 2,000 & 4,000 & 20,000 & 30,000\\
Half-cosine & 0.5 & 500 & 500 & 500 & 500 & 500\\
Half-cosine & 1 & 500 & 500 & 500 & 500 & 500\\
Half-cosine & 5 & 1,000 & 1,000 & 5,000 & 15,000 & 15,000\\
Concave quadratic & 0.5 & 500 & 500 & 500 & 1,000 & 500\\
Concave quadratic & 1 & 500 & 500 & 1,000 & 1,000 & 5,000\\
Concave quadratic & 5 & 1,000 & 5,000 & 15,000 & 15,000 & 50,000\\
\bottomrule
\end{tabular}
\end{table}

\section{Additional selection results}
\label{app:family-metrics}
\subsection{Direct warmup-scaling baseline}
\label{app:direct-duration-scaling}
A simpler alternative is to extrapolate the measured optimal durations
directly. We fit
\begin{equation}
 \widehat W(T)=\min\{T,\;c(T/32\mathrm{k})^\alpha\}
\label{eq:direct-policy}
\end{equation}
and map the prediction to the nearest available warmup. This baseline tests
whether modeling the loss surface provides information beyond the observed
duration trend.

\begin{table}[!htbp]
\centering
\setlength{\tabcolsep}{15pt}
\caption{\textbf{Selection results including the direct warmup-scaling baseline.}
All methods are evaluated on the same retained candidate sets. For the tuned rules, the alternating column would measure one checkpoint ahead rather than extrapolation, so it is omitted.}
\label{tab:learned-controls}
\small
\begin{tabular}{@{}lcccc@{}}
\toprule
Selector & Parameters & Alternating & Four horizons & 128k\\
\midrule
Absolute fit          & 6 & $2.05 \pm 1.61$          & $2.26 \pm 0.66$          & 3.05\\
Difference fit        & 5 & $1.45 \pm 2.22$          & $\mathbf{1.58} \pm 0.77$ & 2.61\\
Direct warmup scaling & 2 & $\mathbf{1.34} \pm 4.94$ & $1.99 \pm 0.69$          & 2.86\\
Fixed best duration & 0 & -- & $3.00\pm1.75$ & 3.87\\
Fixed best fraction & 0 & -- & $1.73\pm0.28$ & \textbf{2.06}\\
1k target             & 0 & $15.47 \pm 10.80$        & $13.79 \pm 0.22$         & 13.56\\
10\% target           & 0 & $13.39 \pm 4.10$         & $12.58 \pm 0.90$         & 11.72\\
\bottomrule
\end{tabular}
\end{table}

With the full warmup grid, extrapolating the measured optimum directly is competitive with the loss-law fits, as are the tuned rules of \Cref{app:calibration}, because the grid supplies accurate pairs of horizon and optimal warmup. That information is unavailable in the three-run setting, where all three rules incur several times the regret of the absolute fit (\Cref{tab:practical-extrapolation-complete}). Their anchor is restricted to the three fitted durations, whereas the loss law uses every post-warmup checkpoint and can recommend durations between and beyond them.

\subsection{Near-optimal measured warmups}
\label{app:near-optimal-warmups}

The best tested warmup can make the optimum appear more sharply determined
than the measured loss differences imply. For each family and target horizon,
we therefore also consider the tested durations whose loss is within
$\epsilon$ of the best observed loss,
\[
 \mathcal S_\epsilon(T)
 =
 \left\{
 W\in\mathcal G_T:
 L(W,T)-\min_{u\in\mathcal G_T}L(u,T)\le\epsilon
 \right\}.
\]
The set contains only measured warmups; gaps between tested durations are not
filled by interpolation. \Cref{tab:near-optimal-warmups} shows that several measured durations are
often effectively competitive. The sets expand as the loss tolerance increases. This helps explain why two selectors can choose visibly different
warmups while incurring similar regret.

\begin{table}[!htbp]
\centering
\setlength{\tabcolsep}{15pt}
\caption{\textbf{Near-optimal measured warmups.}
The table summarizes all retained family--horizon comparisons used by the
later-horizon benchmark. ``Multiple'' counts comparisons with more than one
tested duration within the stated loss tolerance.}
\label{tab:near-optimal-warmups}
\small
\begin{tabular}{rrrr}

Tolerance ($10^{-3}$) &
Multiple &
Median count &
Median fraction of grid \\
\midrule
1 & 62 / 132 & 1 & 16.7\%\\
2 & 92 / 132 & 2 & 20.0\%\\
5 & 124 / 132 & 4 & 40.8\%\\
\bottomrule
\end{tabular}
\end{table}

\section{Extrapolation from short runs}
\label{app:additional}
\subsection{Three-run extrapolation}
\label{app:practical-extrapolation}\label{app:three-run}
For each model-size and peak-rate family, we fit the loss law using three
separated warmup trajectories and checkpoints shared after the longest
warmup has ended. The fitted coefficients are frozen before evaluating later horizons.

\begin{table}[!htbp]
\centering
\setlength{\tabcolsep}{15pt}
\caption{\textbf{Three-run extrapolation protocol.}}
\label{tab:three-run-protocol}
\small
\begin{tabular}{ll}
Item & Choice \\\midrule
Target fitting warmups & 2k, 8k, 16k updates \\
Fit horizons & 20k and 32k updates \\
Target horizons & 50k, 75k, 100k, 128k updates \\
Loss-law fits & absolute and difference forms \\
Primary practical predictor & absolute fit \\
Target candidates & same retained warmup grid as the full-grid benchmark \\
\bottomrule
\end{tabular}
\vspace{-1em}
\end{table}

\begin{table}[!htbp]
\centering
\setlength{\tabcolsep}{15pt}
\caption{\textbf{Complete three-run comparison.}
Mean regret on the same later-horizon candidate grids as the main-text
experiment, in $10^{-3}$ loss units.}
\label{tab:practical-extrapolation-complete}
\small
\begin{tabular}{lrrr}
Selector & Through 20k (29 families) & Through 32k (33 families)\\\midrule
Absolute fit & $\boldsymbol{2.83} \pm 0.16$ & $\boldsymbol{2.51} \pm 0.27$  \\
Difference fit & $4.22 \pm 0.29$ & $2.62 \pm 0.36$  \\
Direct warmup scaling & $9.09 \pm 1.25$ & $7.08 \pm 0.82$  \\
Fixed best duration & $9.31\pm1.21$ & $7.61\pm1.12$\\
Fixed best fraction & $11.72\pm0.82$ & $8.47\pm0.62$\\
1k target & $14.28 \pm 0.30$ & $13.79 \pm 0.22$  \\
10\% target & $13.42 \pm 0.96$ & $12.58 \pm 0.90$ \\
\bottomrule
\end{tabular}
\end{table}

\subsection{Recommendations beyond the observed warmups}
\label{app:pilot-extrapolation-range}

\begin{figure}[!htbp]
\centering
\includegraphics[width=0.8\textwidth]{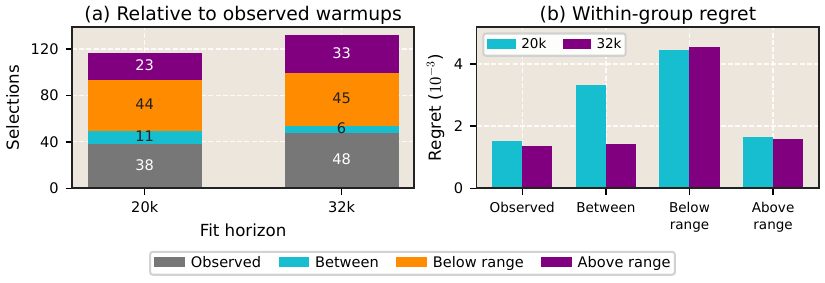}
\caption{\textbf{Where three-run warmup recommendations fall relative to the
observed warmups.}
Left: number of family--horizon selections that coincide with an observed fitting
warmup, fall between observed warmups, or lie below or above their range.
Right: mean measured regret within each group. All selections are frozen
before target losses are scored.}
\label{fig:pilot-support}
\end{figure}

The three-run predictor is not restricted to selecting one of the warmup
durations used for fitting. We group each frozen recommendation according to
whether it coincides with an observed fitting warmup, lies between two observed
warmups, or falls below or above the observed range. Every recommendation is
then scored using the measured target-horizon trajectory at the selected duration.

\Cref{fig:pilot-support} shows that many selections use durations that were not used for fitting. In particular, recommendations above the observed warmup range remain low-regret, so the fitted loss surface can support useful extrapolation in warmup duration as well as in training horizon.

\subsection{Sensitivity to the choice of fitting runs}
\label{app:pilot-placement}

We next test how the warmup range covered by the three fitting runs affects
later-horizon selection. The alternatives either span the available warmup
range broadly or concentrate all three runs near its short, middle, or long
end. We also evaluate randomly chosen triplets. All alternatives are defined
from schedule metadata alone, without using their later-horizon losses.

For the non-default designs, target positions are specified in normalized
$\log(W+w_0)$ coordinates over the available warmup range. The wide design
targets the two ends and midpoint; the clustered designs place all three
targets near the short, middle, or long end. The random baseline samples
admissible triplets uniformly without replacement. All designs use the same fit horizon and fitting procedure and require shared post-warmup
observations. Changing the longest observed warmup also changes the shared checkpoint
window, so this comparison varies both the warmup durations and the available checkpoint window.

\Cref{fig:pilot-placement,tab:pilot-placement} show that the practical
predictor is not tied to the exact default triplet. Broad log-space coverage
performs similarly or better on this retrospective benchmark, and random
triplets remain competitive on average. Concentrating all three fitting warmups near
the long end is substantially less reliable. These results support choosing fitting runs whose warmups span the relevant range.

\begin{figure}[!htbp]
\centering
\includegraphics[width=0.8\textwidth]{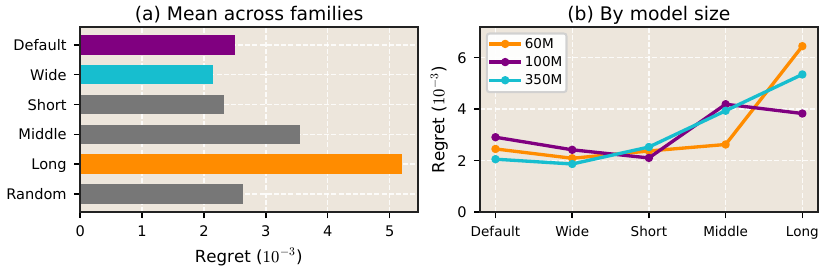}
\caption{\textbf{Sensitivity to the warmup coverage of three fitting runs.}
Left: mean four-horizon regret across families. Right: the same comparison
split by model size. Random-triplet results are first averaged within family,
so families with more admissible triplets receive no additional weight.}
\label{fig:pilot-placement}
\end{figure}

\begin{table}[!htbp]
\centering
\setlength{\tabcolsep}{15pt}
\caption{\textbf{Choice of three fitting runs and selection results.}
All fits use shared post-warmup checkpoints through 32k, with at least four
shared checkpoints. Target positions are normalized coordinates in $\log(W+w_0)$ over each
family's available warmup range. The default row uses the exact warmup
durations from the main experiment. Family values first average regret over
the four target horizons; random triplets are averaged within family before
aggregation. Regret is in $10^{-3}$ loss units.}
\label{tab:pilot-placement}
\footnotesize
\begin{tabular}{llrrr}

Fitting-run design & Log-range positions / rule &
\shortstack{Mean\\regret} & \shortstack{Median\\family} & \shortstack{Worst\\family} \\
\midrule
Default triplet & Main-experiment runs & 2.51 & 1.67 & 8.88\\
Wide coverage & $0,\ 0.5,\ 1$ & 2.15 & 1.47 & 6.80\\
Short-end cluster & $0,\ 0.1,\ 0.2$ & 2.32 & 1.47 & 12.34\\
Middle cluster & $0.4,\ 0.5,\ 0.6$ & 3.55 & 2.13 & 13.44\\
Long-end cluster & $0.8,\ 0.9,\ 1$ & 5.19 & 1.67 & 49.95\\
Random triplets & Up to 8 admissible triplets & 2.63 & 2.21 & 6.72\\
\bottomrule
\end{tabular}
\end{table}

\subsection{Sensitivity to the number of fitting runs}
\label{app:pilot-count}

The three-run experiment asks how little fitting data can support useful extrapolation. We vary the number of observed warmup trajectories at a fixed 32k fit horizon. Unlike the main three-run protocol (\cref{app:three-run}), which uses only checkpoints shared by all three fitting runs, this sweep uses all post-warmup observations from each selected trajectory, so the number of fitting runs can vary across the available grid. \Cref{fig:pilot-count} shows a clear benefit from moving beyond very sparse coverage, followed by smaller, nonmonotonic changes as the observed grid becomes denser. The effect depends on which trajectories are added, but the same pattern appears across model sizes. Three fitting runs are thus a useful low-budget operating point. \Cref{tab:pilot-count-reference} records the corresponding reference values.

\begin{figure}[!htbp]
\centering
\includegraphics[width=0.8\textwidth]{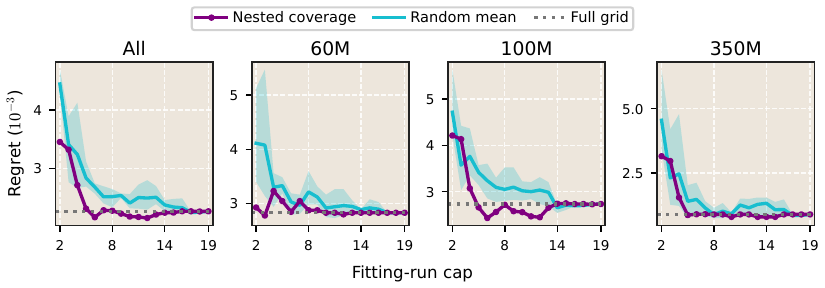}
\caption{\textbf{Selection as the number of fitting runs increases.}
At a cap $n$, each family uses up to $n$ available trajectories.
Fitting uses checkpoints through 32k. Purple shows the nested coverage
order; cyan shows the mean of three fixed random orders, with shading spanning
their aggregate curves. Dotted lines show the full-grid reference.
The same 33 families are used throughout; their available counts range from
3 to 19. Regret averages the four later target horizons.}
\label{fig:pilot-count}
\end{figure}

\begin{table}[!htbp]
\centering
\setlength{\tabcolsep}{15pt}
\caption{\textbf{Reference fits for varying the number of runs.}
The main three-run protocol uses shared post-warmup checkpoints, while the
count sweep uses all post-warmup observations from each selected trajectory.
The full-grid row reproduces the corresponding full-prefix benchmark.
Regret is in $10^{-3}$ loss units.}
\label{tab:pilot-count-reference}
\small
\begin{tabular}{lclr}

Reference & Fitting runs & Checkpoints used & Mean regret \\
\midrule
Main three-run fit & 3 & Shared post-warmup & 2.51\\
Count-sweep fit & 3 & All selected post-warmup & 3.32\\
Full-grid fit & All available & All post-warmup & 2.26\\
\bottomrule
\end{tabular}
\end{table}

\subsection{Fit horizon and short-run budget}
\label{app:pilot-budget}

\begin{figure}[!htbp]
\centering
\includegraphics[width=0.7\textwidth]{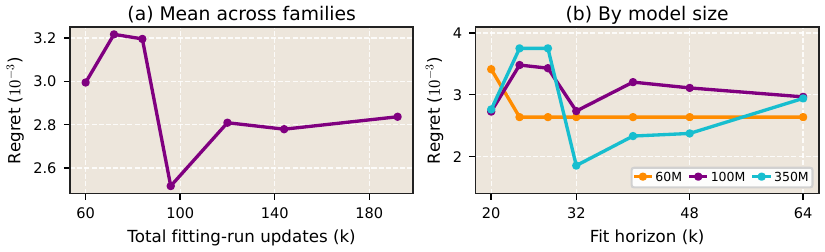}
\caption{\textbf{Sensitivity to fit horizon and short-run budget.}
The same three fitting trajectories are used at every fit horizon, from 20k to 64k.
All comparisons use the same 29 families and the 128k target horizon.
Total fitting-run updates equal three times the fit horizon, excluding the target run. Left: mean final-horizon
regret against the total updates observed across the fitting runs. Right: the same
comparison by model size as the fit horizon increases.}
\label{fig:pilot-budget}
\end{figure}

We next vary how long the same three fitting trajectories are observed before
the fitted coefficients are frozen. All fit horizons are evaluated on the same
families and at the same final target horizon, so the comparison changes
only the amount of short-run data available to the fit.

\Cref{fig:pilot-budget} shows that a longer fit horizon does not produce a
monotonic improvement. The error changes most when the available prefix is
still short and then remains in a similar range across later fit horizons. Simply extending every fitting trajectory is therefore not guaranteed to improve
the final selection.

\subsection{Long-horizon extrapolation protocol}
\label{app:long-horizon-extrapolation}\label{app:long-horizon}
The long-horizon experiment repeats the same procedure with the 100M model
at batch size 128 and longer target horizons. For each peak rate, we fit the
loss law independently using three fitting runs; the frozen absolute fit is minimized
continuously over warmup duration at each target horizon.

\begin{table}[!htbp]
\centering
\setlength{\tabcolsep}{15pt}
\caption{\textbf{Long-horizon protocol.}}
\label{tab:long-horizon-setup}
\small
\begin{tabular}{ll}
Item & Setting\\\midrule
Model / batch / sequence length & 100M / 128 sequences / 256 tokens\\
Fit horizons & 50k and 75k updates\\
Target horizons & 250k, 500k, and 1M updates\\
Fitting warmups, LR $0.5\times10^{-3}$ & 0, 512, 8k warmup updates\\
Fitting warmups, LR $2\times10^{-3}$ & 8k, 16k, 32k warmup updates\\
Fitting warmups, LR $5\times10^{-3}$ & 12k, 16k, 32k warmup updates\\
Short-warmup baseline & 1k warmup updates at every peak rate\\
Evaluation & validation loss vs.\ 1k and 10\% warmup baselines\\
\bottomrule
\end{tabular}
\end{table}

The predicted schedule is trained from initialization. Because these targets use selected schedules instead of a dense warmup sweep, we report validation loss directly, not regret to an oracle warmup.

\subsection{Model-scale extrapolation}
\label{app:model-scale-extrapolation}

Our focus is the horizon dependence of warmup and extrapolation to longer training runs at a fixed model size. We also observe similar behavior across model scales: a larger model at a lower peak learning rate can exhibit a similar dependence of preferred warmup on the training horizon to a smaller model at a higher rate. 

\begin{figure}[!htbp]
\centering
\includegraphics[width=0.8\textwidth]{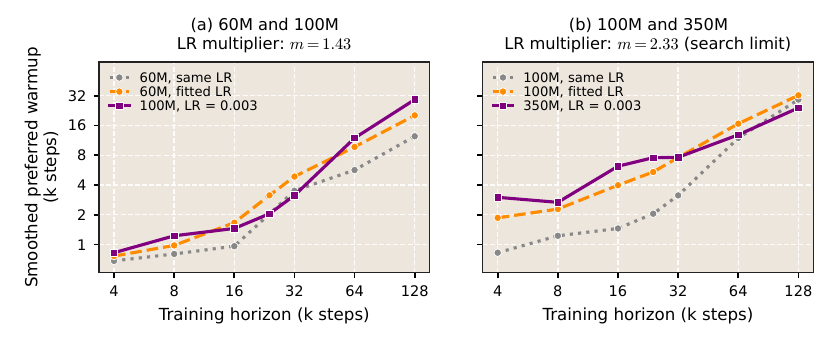}
\caption{\textbf{Warmup behavior across model scales.} The larger model uses peak LR $0.003$; the smaller model uses the same LR (gray) or a fitted higher LR (orange). We average $\log(W+32)$, with $W$ in steps, over seed-0 warmup sweeps with weights proportional to $\exp(-L/0.002)$, then fit monotonically in log LR. The multiplier minimizes weighted squared log-warmup differences across horizons. The 100M-to-350M fit reaches the LR limit, $0.007$.}
\label{fig:scale-warmup-matching}
\end{figure}

Figure~\ref{fig:scale-warmup-matching} illustrates this correspondence. We compare smoothed estimates of the preferred warmup and fit one learning-rate multiplier $m$ per model pair, shared across horizons, with $\eta_{\mathrm{small}}=m\eta_{\mathrm{large}}$. This shift improves the overall agreement, although differences remain at individual horizons. This correspondence suggests that a model-size correction, such as a rescaling of the effective peak learning rate, could extend horizon-based warmup predictions across scales.

\section{Sensitivity of the loss law}
\label{app:robustness}
\subsection{Warmup exponent}
\label{app:fixed-exponent}
We compare the free warmup exponent with a fixed-$s$ ablation at $s=1/2$.

\begin{table}[!htbp]
\centering
\setlength{\tabcolsep}{15pt}
\caption{\textbf{Sensitivity to the warmup exponent.}
RMSE and regret are in $10^{-3}$ loss units. Lower values are better.}
\label{tab:fixed-exponent}
\small
\begin{tabular}{lrr}
Metric & Free $s$ & Fixed $s=1/2$ \\\midrule
Full-data absolute-fit median RMSE & $\boldsymbol{4.92}$ & $5.94$ \\
Held-out absolute-fit median RMSE & $\boldsymbol{4.98}$ & $5.96$ \\
Four-horizon absolute-fit regret & $2.26 \pm 0.66$ & $\boldsymbol{2.00} \pm 0.35$ \\
Three-run absolute-fit regret, 20k (29 families) & $\boldsymbol{2.83} \pm 0.16$ & $2.87 \pm 0.32$ \\
Three-run absolute-fit regret, 32k & $\boldsymbol{2.51} \pm 0.27$ & $2.64 \pm 0.23$ \\
Four-horizon difference-fit regret & $\boldsymbol{1.58} \pm 0.77$ & $\boldsymbol{1.58} \pm 0.67$ \\
Three-run difference-fit regret, 20k (29 families) & $\boldsymbol{4.22} \pm 0.29$ & $4.85 \pm 0.32$ \\
Three-run difference-fit regret, 32k & $\boldsymbol{2.62} \pm 0.36$ & $2.86 \pm 0.21$ \\
\bottomrule
\end{tabular}
\end{table}

The free exponent improves fit quality, while selection performance remains
similar across the reported comparisons.

\subsection{Zero-warmup offset}
\label{app:offset-sensitivity}
We refit the model over a range of fixed offsets in $W+w_0$.

\begin{table}[H]
\centering
\caption{\textbf{Sensitivity to the fixed warmup offset.}
Each row refits all families with the indicated global offset. Here, later-horizon regret averages the dense 33k--128k grid, rather than the four target horizons used in the main comparison. Regret is in $10^{-3}$ loss units.}
\label{tab:offset-difference}\label{tab:offset-absolute}
\small
\begin{tabular}{rrrrr}
& \multicolumn{2}{c}{Difference fit} & \multicolumn{2}{c}{Absolute fit} \\
$w_0$ (updates) & Later-horizon regret & Final-horizon regret & Later-horizon regret & Final-horizon regret \\\midrule
8 & $1.50$ & $2.58$ & $2.32$ & $3.54$ \\
16 & $1.84$ & $2.54$ & $2.09$ & $2.83$ \\
32 & $1.85$ & $2.61$ & $2.49$ & $3.05$ \\
64 & $1.61$ & $2.67$ & $2.41$ & $3.16$ \\
128 & $1.68$ & $2.92$ & $2.28$ & $3.06$ \\
\bottomrule
\end{tabular}
\end{table}

The offset $w_0$ only keeps the law finite at $W=0$, so we fix it rather
than fit it. \Cref{tab:offset-difference} refits every family with offsets from 8 to 128
updates. Across this 16-fold range, later-horizon regret changes by at
most $0.4\times10^{-3}$ for either fitting form, so the fits are
insensitive to the choice; we use $w_0=32$ updates throughout.

\subsection{Sharing the persistent exponent sum across peak rates}
\label{app:shared-exponent-sum}

The local product construction in Appendix~\ref{app:core} motivates
$q_\eta+s_\eta=\gamma$ at a fixed reference point, while the empirical law
fits $q$ and $s$ independently in every peak-rate family. We test an
intermediate constraint by jointly fitting all peak-rate families at each
model size with
\[
 q_\eta+s_\eta=\gamma_N,
\]
where the shared exponent sum $\gamma_N$ is learned from the same fitting
data. The remaining parameters stay family-specific. The joint objective averages
the mean Huber loss equally across peak-rate families and retains the original
parameter bounds. For $J$ peak-rate
families, this reduces the parameter count from $6J$ to $5J+1$.

We evaluate the constraint in two settings. Full-data fits measure descriptive
fit quality. Separately, fits restricted to the early training
window learn both the family-specific parameters and the shared exponent sum
from those prefixes, freeze them, and select warmup at the later target
horizons.

\begin{figure}[!htbp]
\centering
\includegraphics[width=\textwidth]{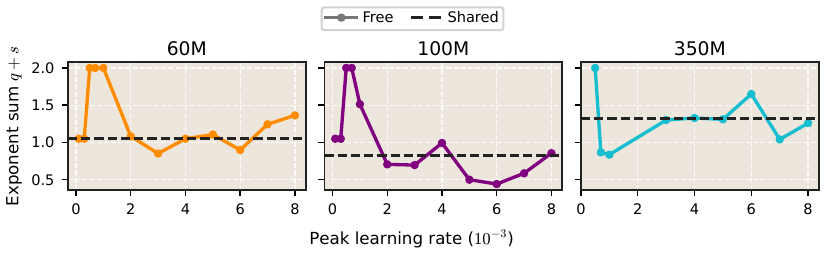}
\caption{\textbf{Freely fitted and shared persistent exponent sums.}
Points show $q_\eta+s_\eta$ from independent full-data fits at each peak
learning rate. Dashed lines show the single sum fitted jointly across all peak-rate families of the same
model size. Points at 2.0 lie on the parameter bound ($q=s=1$).}
\label{fig:shared-exponent-values}
\end{figure}

\Cref{fig:shared-exponent-values} compares the freely fitted sums $q_\eta+s_\eta$
with the shared value $\gamma_N$. For most moderate and large peak rates
the free sums already lie near the shared value. The largest departures
occur at the smallest peak rates, where several free fits reach the upper
bound $q_\eta+s_\eta=2$, that is, $q_\eta=s_\eta=1$. These exponents are set by the bounds rather
than determined by the data, so we do not read them as a change in the
underlying relaxation exponent.

\Cref{fig:shared-exponent-performance,tab:shared-exponent-sum} show that
sharing the sum preserves most of the descriptive fit quality while reducing the exponent parameters. Its effect on later
warmup selection is mixed across model sizes: it improves some groups and
degrades others. We therefore retain independently fitted exponents in the
primary predictor, while the ablation shows that a much more constrained
cross-rate parameterization captures nearly the same observed loss structure.

\begin{figure}[!htbp]
\centering
\includegraphics[width=0.7\textwidth]{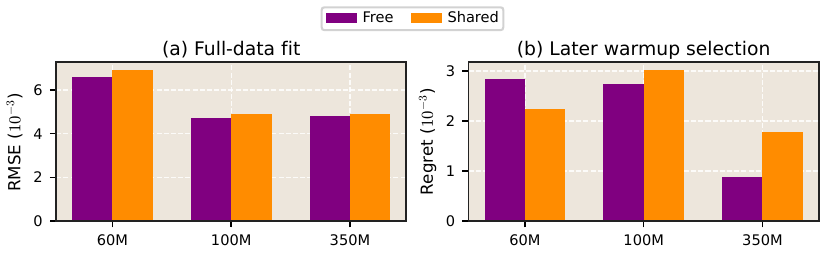}
\caption{\textbf{Effect of sharing $q+s$ across peak learning rates.}
Left: median descriptive RMSE for free and shared-sum fits at each model size.
Right: later-horizon selection regret after fitting only through the
32k fit window. The shared sum is learned from the same fitting data
and is not transferred from the full-data fit.}
\label{fig:shared-exponent-performance}
\end{figure}

\begin{table}[!htbp]
\centering
\setlength{\tabcolsep}{15pt}
\caption{\textbf{Shared-exponent-sum ablation.}
RMSE and regret are in $10^{-3}$ loss units. RMSE and $R^2$ are medians over
peak-rate families. Full-data fit quality is descriptive; later regret averages
families and the four later target horizons after fitting through 32k.}
\label{tab:shared-exponent-sum}
\small
\begin{tabular}{llrrr}

Model & Fit & Full RMSE & Full $R^2$ & Later regret \\
\midrule
60M & Free & \textbf{6.59} & \textbf{0.993} & 2.83\\
60M & Shared & 6.92 & \textbf{0.993} & \textbf{2.24}\\
100M & Free & \textbf{4.71} & \textbf{0.997} & \textbf{2.73}\\
100M & Shared & 4.89 & \textbf{0.997} & 3.02\\
350M & Free & \textbf{4.81} & \textbf{0.997} & \textbf{0.88}\\
350M & Shared & 4.88 & \textbf{0.997} & 1.77\\
\bottomrule
\end{tabular}
\end{table}

\section{Complete fit gallery}
\label{app:atlases}
The plots below show the full-data absolute fit for every analyzed family.
Points are measured validation losses and curves are model predictions at the
displayed horizons.

\label{app:additional}

\begin{figure}[H]
\centering
\includegraphics[width=\textwidth]{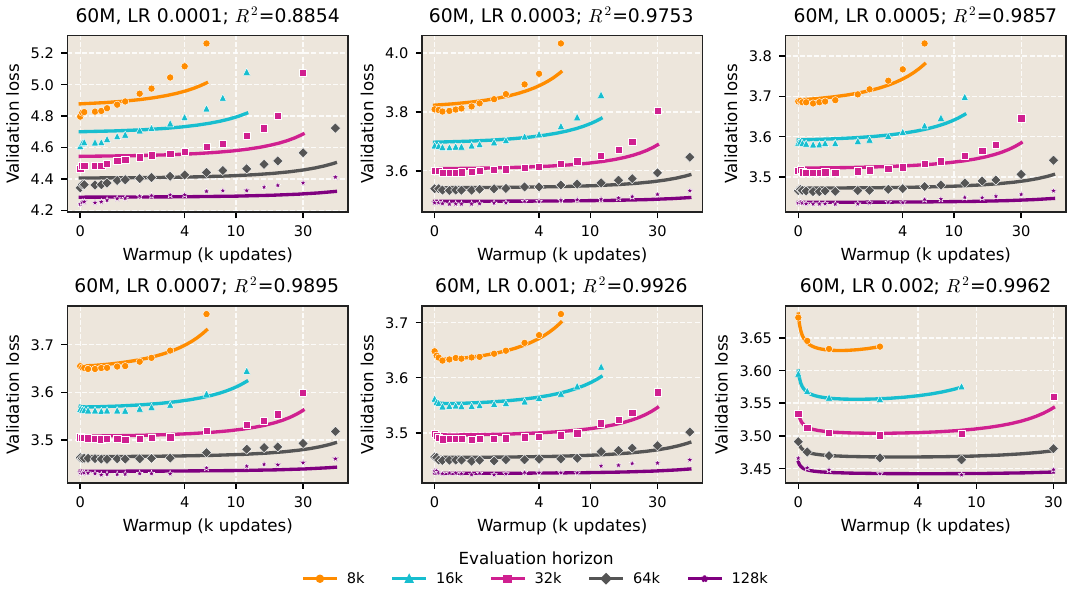}
\caption{Full-data absolute fits: 60M, lower peak learning rates.}
\end{figure}

\begin{figure}[H]
\centering
\includegraphics[width=\textwidth]{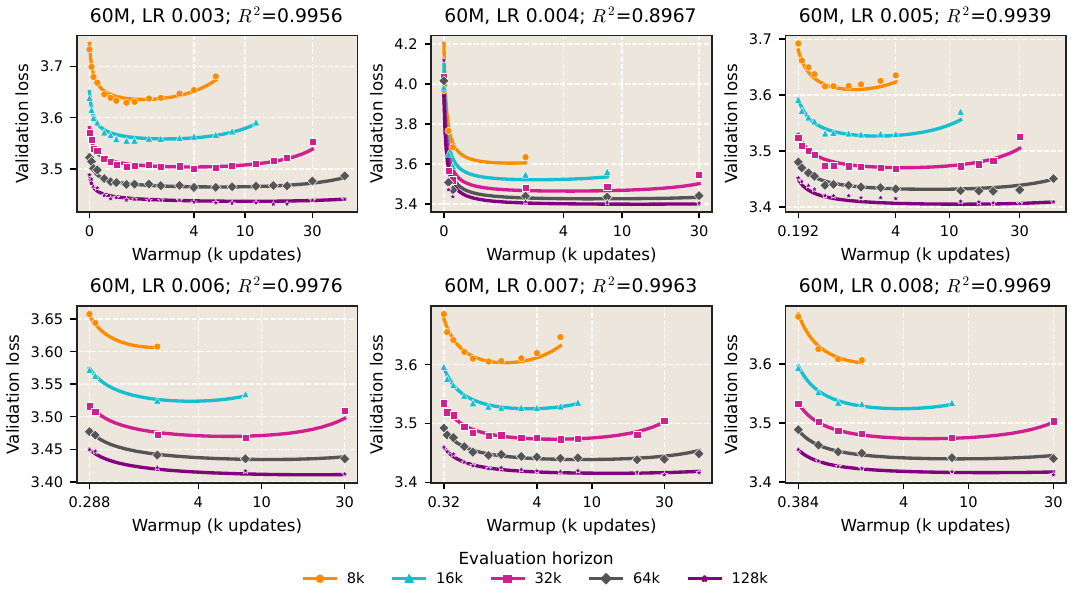}
\caption{Full-data absolute fits: 60M, higher peak learning rates.}
\end{figure}

\begin{figure}[H]
\centering
\includegraphics[width=\textwidth]{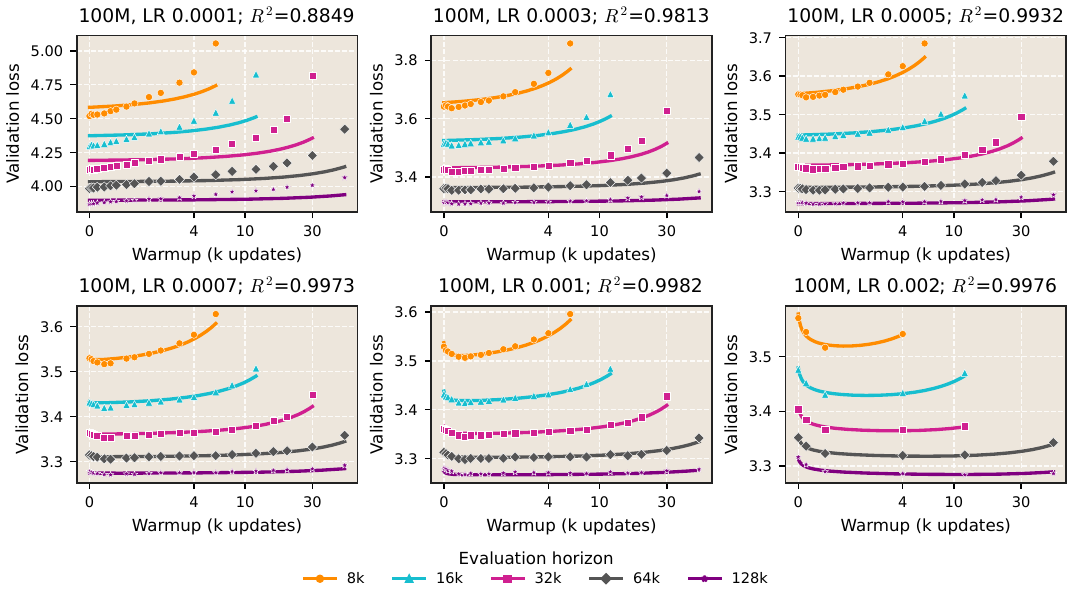}
\caption{Full-data absolute fits: 100M, lower peak learning rates.}
\end{figure}

\begin{figure}[H]
\centering
\includegraphics[width=\textwidth]{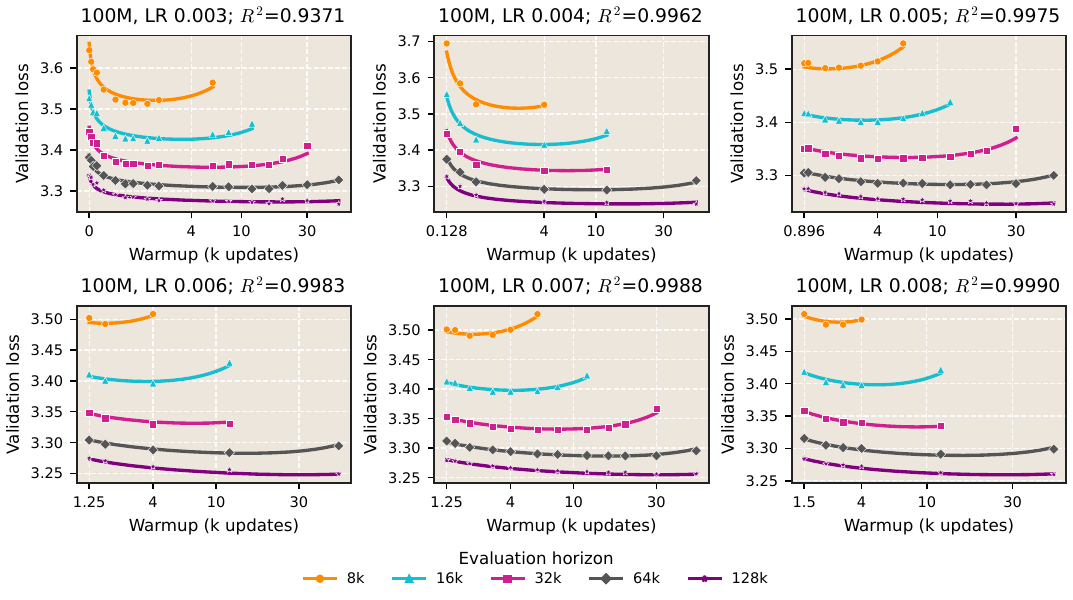}
\caption{Full-data absolute fits: 100M, higher peak learning rates.}
\end{figure}

\begin{figure}[H]
\centering
\includegraphics[width=\textwidth]{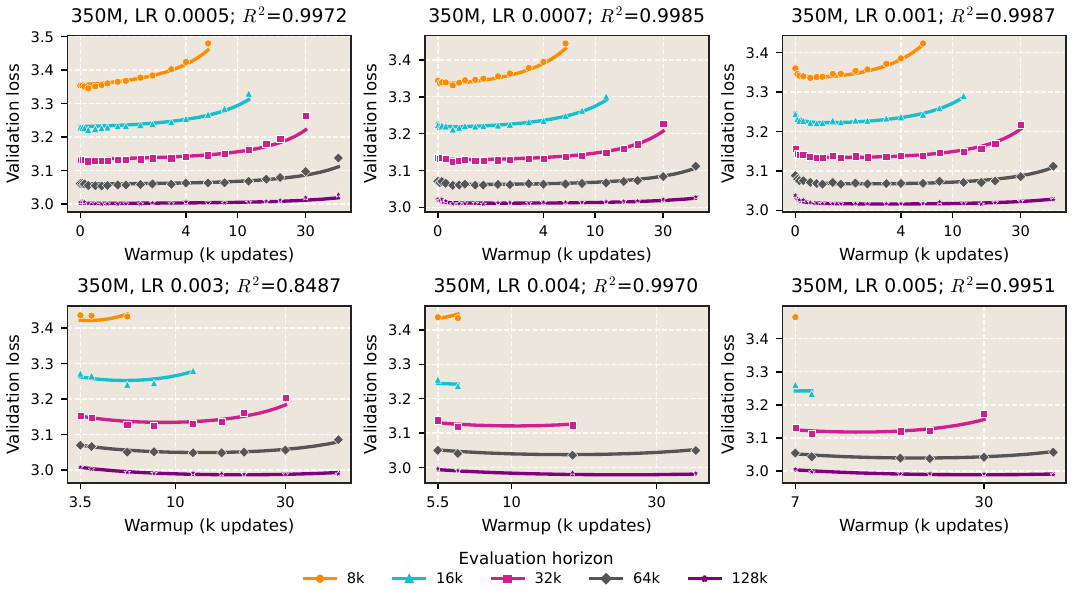}
\caption{Full-data absolute fits: 350M, lower peak learning rates.}
\end{figure}

\begin{figure}[H]
\centering
\includegraphics[width=\textwidth]{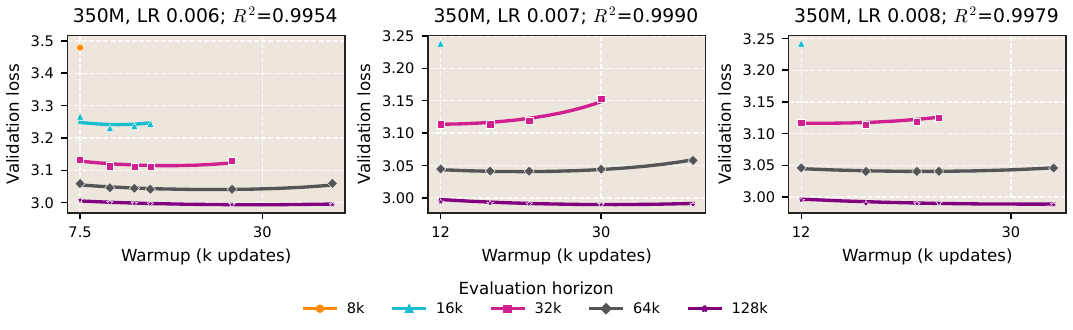}
\caption{Full-data absolute fits: 350M, higher peak learning rates.}
\end{figure}

\section{Related work}
\label{sec:related}
\paragraph{Warmup.}
Learning-rate warmup is widely used in large-batch training and language
models \citep{goyal2017accurate,black2022neox,groeneveld2024olmo}.
Recent large-model recipes use warmups ranging from 500 to 8{,}000 optimizer
steps \citep{grattafiori2024llama3,deepseek2024v3,kimi2025k2}. Prior work has
connected warmup to adaptive-optimizer statistics, update size, curvature,
and early optimization dynamics
\citep{liu2020radam,ma2021adequacy,gilmer2022loss,kalra2024warmup,
kosson2024warmup,alimisis2025warmup,liu2025warmupconvergence}.
Our focus is complementary: we study how the loss-optimal duration changes
with the training horizon and peak learning rate.

\paragraph{Hyperparameter scaling.}
Prior work predicts or transfers hyperparameters such as the learning rate,
batch size, and weight decay across model width, training horizon, and compute
\citep{yang2022mup,bjorck2025horizonlr,filatov2024time,steplaw2025,
zhang2025critical,bergsma2025powerlines}. Warmup can also shift the
learning-rate optimum \citep{filatov2024time} and affect compute-optimal
comparisons \citep{porian2024discrepancies}. Our experiments keep model
size and peak learning rate fixed within each family and ask how warmup
should change as the horizon grows.

\paragraph{Loss models and quadratic theory.}
Foundational scaling laws relate loss to model size, data, and compute
\citep{kaplan2020scaling,hoffmann2022chinchilla}. Schedule-aware
learning-curve models additionally describe accumulated learning rate and
annealing effects \citep{tissue2024annealing,luo2025mpl}.
Quadratic and least-squares models provide explicit analyses of
learning-rate, batch-size, and spectral effects
\citep{jain2018parallel,jain2018accelerating,zou2023benign,zhang2019nqm,
meterez2025seesaw,meterez2026defense}. Power-law spectral assumptions also
lead naturally to algebraic learning curves
\citep{caponnetto2007optimal,bordelon2020spectrum,canatar2021spectral,
bahri2021explaining,bordelon2024dynamical,velikanov2024tight}. These ideas
motivate the quadratic analysis and fitted horizon-dependent loss law used
in this paper.

%% file: references.bib
@article{goyal2017accurate,
  author  = {Goyal, Priya and Doll{\'a}r, Piotr and Girshick, Ross and Noordhuis, Pieter and Wesolowski, Lukasz and Kyrola, Aapo and Tulloch, Andrew and Jia, Yangqing and He, Kaiming},
  title   = {Accurate, Large Minibatch {SGD}: Training {ImageNet} in 1 Hour},
  journal = {arXiv preprint arXiv:1706.02677},
  year    = {2017},
  url = {https://arxiv.org/abs/1706.02677}
}

@article{grattafiori2024llama3,
  author = {Grattafiori, Aaron and Dubey, Abhimanyu and Jauhri, Abhinav and others},
  title = {The {Llama} 3 Herd of Models},
  journal = {arXiv preprint arXiv:2407.21783},
  year = {2024},
  url = {https://arxiv.org/abs/2407.21783}
}

@article{deepseek2024v3,
  author = {{DeepSeek-AI}},
  title = {{DeepSeek-V3} Technical Report},
  journal = {arXiv preprint arXiv:2412.19437},
  year = {2024},
  url = {https://arxiv.org/abs/2412.19437}
}

@article{kimi2025k2,
  author = {{Kimi Team}},
  title = {Kimi {K2}: Open Agentic Intelligence},
  journal = {arXiv preprint arXiv:2507.20534},
  year = {2025},
  url = {https://arxiv.org/abs/2507.20534}
}

@article{meterez2026defense,
  author = {Meterez, Alexandru and Nair, Pranav Ajit and Morwani, Depen and Pehlevan, Cengiz and Kakade, Sham and Damian, Alex},
  title = {A Defense of the Quadratic Model},
  journal = {arXiv preprint arXiv:2607.21716},
  year = {2026},
  url = {https://arxiv.org/abs/2607.21716}
}

@inproceedings{meterez2025seesaw,
  author = {Meterez, Alexandru and Morwani, Depen and Wu, Jingfeng and Oncescu, Costin-Andrei and Pehlevan, Cengiz and Kakade, Sham},
  title = {Seesaw: Accelerating Training by Balancing Learning Rate and Batch Size Scheduling},
  booktitle = {International Conference on Learning Representations (ICLR)},
  year = {2026},
  note = {arXiv:2510.14717},
  url = {https://proceedings.iclr.cc/paper_files/paper/2026/file/187f9ba4cd5a59477305ac712282a1cd-Paper-Conference.pdf}
}

@inproceedings{zhang2025critical,
  author = {Zhang, Hanlin and Morwani, Depen and Vyas, Nikhil and Wu, Jingfeng and Zou, Difan and Ghai, Udaya and Foster, Dean and Kakade, Sham},
  title = {How Does Critical Batch Size Scale in Pre-training?},
  booktitle = {International Conference on Learning Representations (ICLR)},
  year = {2025},
  url = {https://proceedings.iclr.cc/paper_files/paper/2025/hash/a6f14f95d9c9443927638bcd5d917a7a-Abstract-Conference.html}
}

@inproceedings{bergsma2025powerlines,
  author = {Bergsma, Shane and Dey, Nolan and Gosal, Gurpreet and Gray, Gavia and Soboleva, Daria and Hestness, Joel},
  title = {Power Lines: Scaling Laws for Weight Decay and Batch Size in {LLM} Pre-training},
  booktitle = {Advances in Neural Information Processing Systems (NeurIPS)},
  volume = {38},
  year = {2025},
  doi = {10.52202/085713-4171},
  url = {https://proceedings.neurips.cc/paper_files/paper/2025/hash/b5f78a17a94da3e34c935515d1b6adae-Abstract-Conference.html}
}

@inproceedings{black2022neox,
  author  = {Black, Sidney and Biderman, Stella and Hallahan, Eric and Anthony, Quentin and Gao, Leo and Golding, Laurence and He, Horace and Leahy, Connor and McDonell, Kyle and Phang, Jason and Pieler, Michael and Prashanth, Usvsn Sai and Purohit, Shivanshu and Reynolds, Laria and Tow, Jonathan and Wang, Ben and Weinbach, Samuel},
  title   = {{GPT-NeoX-20B}: An Open-Source Autoregressive Language Model},
  year    = {2022},
  booktitle = {Proceedings of BigScience Episode {\#}5 -- Workshop on Challenges {\&} Perspectives in Creating Large Language Models},
  pages = {95--136},
  doi = {10.18653/v1/2022.bigscience-1.9},
  url = {https://aclanthology.org/2022.bigscience-1.9/}
}

@inproceedings{groeneveld2024olmo,
  author  = {Groeneveld, Dirk and Beltagy, Iz and Walsh, Evan and Bhagia, Akshita and Kinney, Rodney and Tafjord, Oyvind and Jha, Ananya Harsh and Ivison, Hamish and Magnusson, Ian and Wang, Yizhong and Arora, Shane and Atkinson, David and Authur, Russell and Chandu, Khyathi Raghavi and Cohan, Arman and Dumas, Jennifer and Elazar, Yanai and Gu, Yuling and Hessel, Jack and Khot, Tushar and Merrill, William and Morrison, Jacob and Muennighoff, Niklas and Naik, Aakanksha and Nam, Crystal and Peters, Matthew E. and Pyatkin, Valentina and Ravichander, Abhilasha and Schwenk, Dustin and Shah, Saurabh and Smith, Will and Strubell, Emma and Subramani, Nishant and Wortsman, Mitchell and Dasigi, Pradeep and Lambert, Nathan and Richardson, Kyle and Zettlemoyer, Luke and Dodge, Jesse and Lo, Kyle and Soldaini, Luca and Smith, Noah A. and Hajishirzi, Hannaneh},
  title   = {{OLMo}: Accelerating the Science of Language Models},
  year    = {2024},
  booktitle = {Proceedings of the 62nd Annual Meeting of the Association for Computational Linguistics (Volume 1: Long Papers)},
  pages = {15789--15809},
  doi = {10.18653/v1/2024.acl-long.841},
  url = {https://aclanthology.org/2024.acl-long.841/}
}

@inproceedings{liu2020radam,
  author    = {Liu, Liyuan and Jiang, Haoming and He, Pengcheng and Chen, Weizhu and Liu, Xiaodong and Gao, Jianfeng and Han, Jiawei},
  title     = {On the Variance of the Adaptive Learning Rate and Beyond},
  booktitle = {International Conference on Learning Representations (ICLR)},
  year      = {2020},
  note      = {arXiv:1908.03265},
  url = {https://arxiv.org/abs/1908.03265}
}

@inproceedings{ma2021adequacy,
  author    = {Ma, Jerry and Yarats, Denis},
  title     = {On the Adequacy of Untuned Warmup for Adaptive Optimization},
  booktitle = {AAAI Conference on Artificial Intelligence},
  year      = {2021},
  note      = {arXiv:1910.04209},
  volume = {35},
  pages = {8828--8836},
  doi = {10.1609/aaai.v35i10.17069},
  url = {https://ojs.aaai.org/index.php/AAAI/article/view/17069}
}

@article{liu2025warmupconvergence,
  author  = {Liu, Yuxing and Ge, Yuze and Pan, Rui and An, Kang and Zhang, Tong},
  title   = {Theoretical Analysis on How Learning Rate Warmup Accelerates Convergence},
  journal = {arXiv preprint arXiv:2509.07972},
  year    = {2025},
  url = {https://arxiv.org/abs/2509.07972}
}

@article{riabinin2026warmup,
  author  = {Riabinin, Artem and Veprikov, Andrey and Bolatov, Arman and Tak{\'a}{\v{c}}, Martin and Beznosikov, Aleksandr},
  title   = {Where Does Warm-Up Come From? Adaptive Scheduling for Norm-Constrained Optimizers},
  journal = {arXiv preprint arXiv:2602.05813},
  year    = {2026},
  url = {https://arxiv.org/abs/2602.05813}
}

@inproceedings{bjorck2025horizonlr,
  author    = {Bjorck, Johan and Benhaim, Alon and Chaudhary, Vishrav and Wei, Furu and Song, Xia},
  title     = {Scaling Optimal {LR} Across Token Horizons},
  booktitle = {International Conference on Learning Representations (ICLR)},
  year      = {2025},
  note      = {arXiv:2409.19913},
  url = {https://proceedings.iclr.cc/paper_files/paper/2025/hash/cffa22c56c0df3b3edb1df8a9ad67804-Abstract-Conference.html}
}

@inproceedings{gilmer2022loss,
  author    = {Gilmer, Justin and Ghorbani, Behrooz and Garg, Ankush and Kudugunta, Sneha and Neyshabur, Behnam and Cardoze, David and Dahl, George and Nado, Zachary and Firat, Orhan},
  title     = {A Loss Curvature Perspective on Training Instability in Deep Learning},
  booktitle = {International Conference on Learning Representations (ICLR)},
  year      = {2022},
  note      = {arXiv:2110.04369},
  url = {https://openreview.net/forum?id=OcKMT-36vUs}
}

@inproceedings{kalra2024warmup,
  author    = {Kalra, Dayal Singh and Barkeshli, Maissam},
  title     = {Why Warmup the Learning Rate? {U}nderlying Mechanisms and Improvements},
  booktitle = {Advances in Neural Information Processing Systems (NeurIPS)},
  year      = {2024},
  note      = {arXiv:2406.09405},
  url = {https://arxiv.org/abs/2406.09405}
}

@inproceedings{kosson2024warmup,
  author    = {Kosson, Atli and Messmer, Bettina and Jaggi, Martin},
  title     = {Analyzing \& Reducing the Need for Learning Rate Warmup in {GPT} Training},
  booktitle = {Advances in Neural Information Processing Systems (NeurIPS)},
  year      = {2024},
  note      = {arXiv:2410.23922},
  url = {https://arxiv.org/abs/2410.23922}
}

@inproceedings{alimisis2025warmup,
  author  = {Alimisis, Foivos and Islamov, Rustem and Lucchi, Aurelien},
  title   = {Why Do We Need Warm-up? {A} Theoretical Perspective},
  booktitle = {International Conference on Machine Learning (ICML)},
  year    = {2026},
  note    = {arXiv:2510.03164},
  url = {https://openreview.net/forum?id=a6fo32UnpU}
}

@inproceedings{wen2025river,
  author    = {Wen, Kaiyue and Li, Zhiyuan and Wang, Jason and Hall, David and Liang, Percy and Ma, Tengyu},
  title     = {Understanding Warmup-Stable-Decay Learning Rates: A River Valley Loss Landscape View},
  booktitle = {International Conference on Learning Representations (ICLR)},
  year      = {2025},
  note      = {arXiv:2410.05192},
  url = {https://proceedings.iclr.cc/paper_files/paper/2025/hash/6a1fe80a9e2dcda0b3e5fd0fd87eb097-Abstract-Conference.html}
}

@inproceedings{hagele2024beyond,
  author    = {H{\"a}gele, Alexander and Bakouch, Elie and Kosson, Atli and Ben Allal, Loubna and Von Werra, Leandro and Jaggi, Martin},
  title     = {Scaling Laws and Compute-Optimal Training Beyond Fixed Training Durations},
  booktitle = {Advances in Neural Information Processing Systems (NeurIPS)},
  year      = {2024},
  note      = {arXiv:2405.18392},
  url = {https://arxiv.org/abs/2405.18392}
}

@inproceedings{porian2024discrepancies,
  author    = {Porian, Tomer and Wortsman, Mitchell and Jitsev, Jenia and Schmidt, Ludwig and Carmon, Yair},
  title     = {Resolving Discrepancies in Compute-Optimal Scaling of Language Models},
  booktitle = {Advances in Neural Information Processing Systems (NeurIPS)},
  year      = {2024},
  note      = {arXiv:2406.19146},
  url = {https://arxiv.org/abs/2406.19146}
}

@article{filatov2024time,
  author  = {Filatov, Oleg and Ebert, Jan and Wang, Jiangtao and Kesselheim, Stefan},
  title   = {Time Transfer: On Optimal Learning Rate and Batch Size in the Infinite Data Limit},
  journal = {arXiv preprint arXiv:2410.05838},
  year    = {2024},
  url = {https://arxiv.org/abs/2410.05838}
}

@inproceedings{tissue2024annealing,
  author  = {Tissue, Howe and Wang, Venus and Wang, Lu},
  title   = {Scaling Law with Learning Rate Annealing},
  booktitle = {Advances in Neural Information Processing Systems (NeurIPS)},
  volume  = {38},
  year    = {2025},
  doi = {10.52202/085713-3044},
  url = {https://proceedings.neurips.cc/paper_files/paper/2025/hash/830b1abc6d2da85f23d41169fa44d185-Abstract-Conference.html}
}

@inproceedings{luo2025mpl,
  author    = {Luo, Kairong and Wen, Haodong and Hu, Shengding and Sun, Zhenbo and Liu, Zhiyuan and Sun, Maosong and Lyu, Kaifeng and Chen, Wenguang},
  title     = {A Multi-Power Law for Loss Curve Prediction Across Learning Rate Schedules},
  booktitle = {International Conference on Learning Representations (ICLR)},
  year      = {2025},
  note      = {arXiv:2503.12811},
  url = {https://arxiv.org/abs/2503.12811}
}

@inproceedings{schaipp2025surprising,
  author    = {Schaipp, Fabian and H{\"a}gele, Alexander and Taylor, Adrien and Simsekli, Umut and Bach, Francis},
  title     = {The Surprising Agreement Between Convex Optimization Theory and Learning-Rate Scheduling for Large Model Training},
  booktitle = {International Conference on Machine Learning (ICML)},
  year      = {2025},
  note      = {arXiv:2501.18965},
  series = {Proceedings of Machine Learning Research},
  volume = {267},
  pages = {53267--53294},
  url = {https://proceedings.mlr.press/v267/schaipp25a.html}
}

@article{steplaw2025,
  author  = {Li, Houyi and Zheng, Wenzhen and Wang, Qiufeng and Zhang, Hanshan and Wang, Zili and Xuyang, Shijie and Fan, Yuantao and Ding, Zhenyu and Wang, Haoying and Ding, Ning and Zhou, Shuigeng and Zhang, Xiangyu and Jiang, Daxin},
  title   = {Predictable Scale: {P}art {I}, {S}tep {L}aw -- Optimal Hyperparameter Scaling Law in Large Language Model Pretraining},
  journal = {arXiv preprint arXiv:2503.04715},
  year    = {2025},
  url = {https://arxiv.org/abs/2503.04715}
}

@inproceedings{yang2022mup,
  author  = {Yang, Greg and Hu, Edward J. and Babuschkin, Igor and Sidor, Szymon and Liu, Xiaodong and Farhi, David and Ryder, Nick and Pachocki, Jakub and Chen, Weizhu and Gao, Jianfeng},
  title   = {Tensor Programs {V}: Tuning Large Neural Networks via Zero-Shot Hyperparameter Transfer},
  year    = {2021},
  booktitle = {Advances in Neural Information Processing Systems (NeurIPS)},
  url = {https://proceedings.neurips.cc/paper/2021/hash/8df7c2e3c3c3be098ef7b382bd2c37ba-Abstract.html}
}

@article{kaplan2020scaling,
  author  = {Kaplan, Jared and McCandlish, Sam and Henighan, Tom and Brown, Tom B. and Chess, Benjamin and Child, Rewon and Gray, Scott and Radford, Alec and Wu, Jeffrey and Amodei, Dario},
  title   = {Scaling Laws for Neural Language Models},
  journal = {arXiv preprint arXiv:2001.08361},
  year    = {2020},
  url = {https://arxiv.org/abs/2001.08361}
}

@inproceedings{hoffmann2022chinchilla,
  author  = {Hoffmann, Jordan and Borgeaud, Sebastian and Mensch, Arthur and Buchatskaya, Elena and Cai, Trevor and Rutherford, Eliza and de Las Casas, Diego and Hendricks, Lisa Anne and Welbl, Johannes and Clark, Aidan and Hennigan, Thomas and Noland, Eric and Millican, Katherine and van den Driessche, George and Damoc, Bogdan and Guy, Aurelia and Osindero, Simon and Simonyan, Kar{\'e}n and Elsen, Erich and Vinyals, Oriol and Rae, Jack and Sifre, Laurent},
  title   = {An Empirical Analysis of Compute-Optimal Large Language Model Training},
  year    = {2022},
  booktitle = {Advances in Neural Information Processing Systems (NeurIPS)},
  volume = {35},
  url = {https://proceedings.neurips.cc/paper_files/paper/2022/hash/c1e2faff6f588870935f114ebe04a3e5-Abstract.html}
}

@article{caponnetto2007optimal,
  author  = {Caponnetto, Andrea and De Vito, Ernesto},
  title   = {Optimal Rates for the Regularized Least-Squares Algorithm},
  journal = {Foundations of Computational Mathematics},
  volume  = {7},
  number  = {3},
  pages   = {331--368},
  year    = {2007}
}

@inproceedings{bordelon2020spectrum,
  author    = {Bordelon, Blake and Canatar, Abdulkadir and Pehlevan, Cengiz},
  title     = {Spectrum Dependent Learning Curves in Kernel Regression and Wide Neural Networks},
  booktitle = {International Conference on Machine Learning (ICML)},
  year      = {2020},
  note      = {arXiv:2002.02561},
  series = {Proceedings of Machine Learning Research},
  volume = {119},
  pages = {1024--1034},
  url = {https://proceedings.mlr.press/v119/bordelon20a.html}
}

@article{canatar2021spectral,
  author  = {Canatar, Abdulkadir and Bordelon, Blake and Pehlevan, Cengiz},
  title   = {Spectral Bias and Task-Model Alignment Explain Generalization in Kernel Regression and Infinitely Wide Neural Networks},
  journal = {Nature Communications},
  volume  = {12},
  pages   = {2914},
  year    = {2021},
  doi = {10.1038/s41467-021-23103-1},
  url = {https://www.nature.com/articles/s41467-021-23103-1}
}

@article{bahri2021explaining,
  author  = {Bahri, Yasaman and Dyer, Ethan and Kaplan, Jared and Lee, Jaehoon and Sharma, Utkarsh},
  title   = {Explaining Neural Scaling Laws},
  journal = {Proceedings of the National Academy of Sciences},
  volume  = {121},
  number  = {27},
  pages   = {e2311878121},
  year    = {2024},
  doi     = {10.1073/pnas.2311878121},
  url     = {https://doi.org/10.1073/pnas.2311878121}
}

@inproceedings{bordelon2024dynamical,
  author    = {Bordelon, Blake and Atanasov, Alexander and Pehlevan, Cengiz},
  title     = {A Dynamical Model of Neural Scaling Laws},
  booktitle = {International Conference on Machine Learning (ICML)},
  year      = {2024},
  note      = {arXiv:2402.01092},
  series = {Proceedings of Machine Learning Research},
  volume = {235},
  pages = {4345--4382},
  url = {https://proceedings.mlr.press/v235/bordelon24a.html}
}

@article{velikanov2024tight,
  author  = {Velikanov, Maksim and Yarotsky, Dmitry},
  title   = {Tight Convergence Rate Bounds for Optimization Under Power Law Spectral Conditions},
  journal = {Journal of Machine Learning Research},
  volume  = {25},
  number  = {81},
  pages   = {1--78},
  year    = {2024},
  url = {https://jmlr.org/papers/v25/23-0698.html}
}

@inproceedings{zhang2019nqm,
  author    = {Zhang, Guodong and Li, Lala and Nado, Zachary and Martens, James and Sachdeva, Sushant and Dahl, George and Shallue, Christopher and Grosse, Roger B.},
  title     = {Which Algorithmic Choices Matter at Which Batch Sizes? {I}nsights From a Noisy Quadratic Model},
  booktitle = {Advances in Neural Information Processing Systems (NeurIPS)},
  year      = {2019},
  note      = {arXiv:1907.04164},
  url = {https://arxiv.org/abs/1907.04164}
}

@article{jain2018parallel,
  author  = {Jain, Prateek and Kakade, Sham M. and Kidambi, Rahul and Netrapalli, Praneeth and Sidford, Aaron},
  title   = {Parallelizing Stochastic Gradient Descent for Least Squares Regression: Mini-batching, Averaging, and Model Misspecification},
  journal = {Journal of Machine Learning Research},
  volume  = {18},
  number  = {223},
  pages   = {1--42},
  year    = {2018},
  url = {https://jmlr.org/papers/v18/16-595.html}
}

@inproceedings{jain2018accelerating,
  author    = {Jain, Prateek and Kakade, Sham M. and Kidambi, Rahul and Netrapalli, Praneeth and Sidford, Aaron},
  title     = {Accelerating Stochastic Gradient Descent for Least Squares Regression},
  booktitle = {Conference on Learning Theory (COLT)},
  series    = {Proceedings of Machine Learning Research},
  volume    = {75},
  pages     = {545--604},
  year      = {2018}
}

@article{zou2023benign,
  author  = {Zou, Difan and Wu, Jingfeng and Braverman, Vladimir and Gu, Quanquan and Kakade, Sham M.},
  title   = {Benign Overfitting of Constant-Stepsize {SGD} for Linear Regression},
  journal = {Journal of Machine Learning Research},
  volume  = {24},
  number  = {326},
  pages   = {1--58},
  year    = {2023},
  url = {https://jmlr.org/papers/v24/21-1297.html}
}

@inproceedings{cohen2021eos,
  author    = {Cohen, Jeremy M. and Kaur, Simran and Li, Yuanzhi and Kolter, J. Zico and Talwalkar, Ameet},
  title     = {Gradient Descent on Neural Networks Typically Occurs at the Edge of Stability},
  booktitle = {International Conference on Learning Representations (ICLR)},
  year      = {2021},
  note      = {arXiv:2103.00065},
  url       = {https://arxiv.org/abs/2103.00065}
}

@inproceedings{damian2023self,
  author    = {Damian, Alex and Nichani, Eshaan and Lee, Jason D.},
  title     = {Self-Stabilization: The Implicit Bias of Gradient Descent at the Edge of Stability},
  booktitle = {International Conference on Learning Representations (ICLR)},
  year      = {2023},
  note      = {arXiv:2209.15594},
  url       = {https://arxiv.org/abs/2209.15594}
}

@inproceedings{sagun2017empirical,
  author  = {Sagun, Levent and Evci, Utku and G{\"u}ney, V. Ugur and Dauphin, Yann and Bottou, L{\'e}on},
  title   = {Empirical Analysis of the {H}essian of Over-Parametrized Neural Networks},
  booktitle = {International Conference on Learning Representations (ICLR), Workshop Track},
  year    = {2018},
  note    = {arXiv:1706.04454},
  url     = {https://iclr.cc/virtual/2018/workshop/563}
}

@inproceedings{ghorbani2019investigation,
  author    = {Ghorbani, Behrooz and Krishnan, Shankar and Xiao, Ying},
  title     = {An Investigation into Neural Net Optimization via {H}essian Eigenvalue Density},
  booktitle = {Proceedings of the 36th International Conference on Machine Learning},
  series    = {Proceedings of Machine Learning Research},
  volume    = {97},
  pages     = {2232--2241},
  year      = {2019},
  url       = {https://proceedings.mlr.press/v97/ghorbani19b.html}
}

@inproceedings{vaswani2017attention,
  author = {Vaswani, Ashish and Shazeer, Noam and Parmar, Niki and Uszkoreit, Jakob and Jones, Llion and Gomez, Aidan N. and Kaiser, {\L}ukasz and Polosukhin, Illia},
  title = {Attention Is All You Need},
  booktitle = {Advances in Neural Information Processing Systems (NeurIPS)},
  year = {2017},
  url = {https://arxiv.org/abs/1706.03762}
}

@inproceedings{gotmare2019closer,
  author = {Gotmare, Akhilesh and Keskar, Nitish Shirish and Xiong, Caiming and Socher, Richard},
  title = {A Closer Look at Deep Learning Heuristics: Learning Rate Restarts, Warmup and Distillation},
  booktitle = {International Conference on Learning Representations (ICLR)},
  year = {2019},
  url = {https://openreview.net/forum?id=r14EOsCqKX}
}

@inproceedings{bordelon2022sgd,
  author = {Bordelon, Blake and Pehlevan, Cengiz},
  title = {Learning Curves for Stochastic Gradient Descent on Structured Features},
  booktitle = {International Conference on Learning Representations (ICLR)},
  year = {2022},
  url = {https://arxiv.org/abs/2106.02713}
}

@article{raffel2020t5,
  author = {Colin Raffel and Noam Shazeer and Adam Roberts and Katherine Lee and Sharan Narang and Michael Matena and Yanqi Zhou and Wei Li and Peter J. Liu},
  title = {Exploring the Limits of Transfer Learning with a Unified Text-to-Text Transformer},
  journal = {Journal of Machine Learning Research},
  year = {2020},
  volume = {21},
  number = {140},
  pages = {1--67},
  url = {https://www.jmlr.org/papers/v21/20-074.html}
}

@misc{jordan2024muon,
  author = {Keller Jordan},
  title = {{Muon}: An Optimizer for Hidden Layers in Neural Networks},
  year = {2024},
  howpublished = {\url{https://kellerjordan.github.io/posts/muon/}}
}

@article{liu2025muon,
  author = {Jingyuan Liu and Jianlin Su and Xingcheng Yao and Zhejun Jiang and Guokun Lai and Yulun Du and Yidao Qin and Weixin Xu and Enzhe Lu and Junjie Yan and Yanru Chen and Huabin Zheng and Yibo Liu and Shaowei Liu and Bohong Yin and Weiran He and Han Zhu and Yuzhi Wang and Jianzhou Wang and Mengnan Dong and Zheng Zhang and Yongsheng Kang and Hao Zhang and Xinran Xu and Yutao Zhang and Yuxin Wu and Xinyu Zhou and Zhilin Yang},
  title = {{Muon} is Scalable for {LLM} Training},
  journal = {arXiv preprint arXiv:2502.16982},
  year = {2025},
  url = {https://arxiv.org/abs/2502.16982}
}

@inproceedings{zhao2024galore,
  author = {Jiawei Zhao and Zhenyu Zhang and Beidi Chen and Zhangyang Wang and Anima Anandkumar and Yuandong Tian},
  title = {{GaLore}: Memory-Efficient {LLM} Training by Gradient Low-Rank Projection},
  booktitle = {Proceedings of the 41st International Conference on Machine Learning},
  series = {Proceedings of Machine Learning Research},
  volume = {235},
  pages = {61121--61143},
  year = {2024},
  publisher = {PMLR},
  url = {https://proceedings.mlr.press/v235/zhao24s.html}
}

@book{olver1974asymptotics,
  author    = {Olver, Frank W. J.},
  title     = {Asymptotics and Special Functions},
  publisher = {Academic Press},
  address   = {New York},
  year      = {1974}
}
